\documentclass{article}
\usepackage[accepted]{icml2022}
\usepackage{natbib}

\usepackage{graphicx}
\usepackage{subfig}

\usepackage{xr}
\usepackage{hyperref}
\usepackage{algorithm,algorithmic}
\usepackage{makecell}
\usepackage{amssymb, multirow, paralist, color,amsmath}
\usepackage{enumitem}
\usepackage{textcase}
\usepackage[T1]{fontenc}
\usepackage{mathtools}
\usepackage{cleveref}
  
\usepackage{empheq}
\usepackage{color}
\definecolor{darkgreen}{rgb}{0.00,0.5,0.00}
\usepackage{fancybox}
\usepackage{tablefootnote}
\usepackage{natbib}
\usepackage{graphicx}
\usepackage{caption}
\usepackage{amsthm}
\newtheorem{thm}{Theorem}
\newtheorem{prop}{Proposition}
\newtheorem{lemma}{Lemma}
\newtheorem{remark}{Remark}

\newtheorem{ass}{Assumption}

\definecolor{myteal}{RGB}{27,158,119}
\definecolor{myorange}{RGB}{217,95,2}
\definecolor{myred}{RGB}{231,41,138}
\definecolor{mypurple}{RGB}{152,78,163}
\definecolor{myblue}{rgb}{.9, .9, 1}
\definecolor{mygreen}{RGB}{0,100,0}
\definecolor{mycyan}{rgb}{0.88,1,1}
\definecolor{mydarkred}{RGB}{192,47,25}
\newcommand*\circled[1]{\tikz[baseline=(char.base)]{
		\node[shape=circle,draw,inner sep=0.3pt] (char) {#1};}}

\def \S {\mathbf{S}}
\def \A {\mathcal{A}}

\def \R {\mathbb{R}}

\def \w {\mathbf{w}}
\def \v {\mathbf{v}}

\def \x {\mathbf{x}}

\def \E {\mathbb{E}}

\def \x {\mathbf{x}}

\def \1 {\mathbf{1}}

\def \y {\mathbf{y}}
\def \u {\mathbf{u}}

\def \y {\mathbf{y}}
\def \u {\mathbf{u}}

\def \F {\mathcal{F}}

\def \B {\mathcalB}

\def \y {\mathbf{y}}
\def \E {\mathbb{E}}
\def \x {\mathbf{x}}

\def \u {\mathbf{u}}

\def \w {\mathbf{w}}

\def \R {\mathbb{R}}
\def \S {\mathcal{S}}

\def \N {\mathcal{N}}
\def \A {\mathcal{A}}

\def \v {\mathbf{v}}

\def \B {\mathcal{B}}

\def \T {\mathcal{T}}
\def \F {\mathscr{F}}

\def \m {\mathbf{m}}

\newcommand{\Norm}[1]{\left\|#1\right\|}

\usepackage[skins]{tcolorbox}

\def \F {\mathcal{F}}

\usepackage{cancel}
\usepackage{colortbl}

\def \B {\mathcalB}

\def \O {\mathcal O}

\usepackage{bbm}
\usepackage{booktabs}
\usepackage{tablefootnote}

\newlength\mytemplen
\newsavebox\mytempbox

\makeatletter
\newcommand\mybluebox{%
	\@ifnextchar[
	{\@mybluebox}%
	{\@mybluebox[0pt]}}

\def\@mybluebox[#1]{%
	\@ifnextchar[
	{\@@mybluebox[#1]}%
	{\@@mybluebox[#1][0pt]}}

\def\@@mybluebox[#1][#2]#3{
	\sbox\mytempbox{#3}%
	\mytemplen\ht\mytempbox
	\advance\mytemplen #1\relax
	\ht\mytempbox\mytemplen
	\mytemplen\dp\mytempbox
	\advance\mytemplen #2\relax
	\dp\mytempbox\mytemplen
	\colorbox{myblue}{\hspace{1em}\usebox{\mytempbox}\hspace{1em}}}

\def \y {\mathbf{y}}
\def \E {\mathbb{E}}

\def \x {\mathbf{x}}

\def \u {\mathbf{u}}

\def \w {\mathbf{w}}

\def \R {\mathbb{R}}
\def \S {\mathcal{S}}

\def \N {\mathcal{N}}
\def \A {\mathcal{A}}

\def \v {\mathbf{v}}

\usepackage{amssymb}
\usepackage{pifont}

\usepackage{bm}

\usepackage[colorinlistoftodos,bordercolor=green,backgroundcolor=green!20,linecolor=green,textsize=scriptsize]{todonotes}

\def \B {\mathcal{B}}

\def \T {\mathcal{T}}
\def \F {\mathcal{F}}

\usepackage{todonotes}

\icmltitlerunning{GraphFM: Improving Large-scale GNN training via Feature Momentum}

\begin{document}

\twocolumn[
\icmltitle{GraphFM: Improving Large-Scale GNN Training via Feature Momentum}


\icmlsetsymbol{equal}{*}

\begin{icmlauthorlist}
\icmlauthor{Haiyang Yu*}{tamu}
\icmlauthor{Limei Wang*}{tamu}
\icmlauthor{Bokun Wang*}{uiowa}
\icmlauthor{Meng Liu}{tamu}
\icmlauthor{Tianbao Yang}{uiowa}
\icmlauthor{Shuiwang Ji}{tamu}
\end{icmlauthorlist}

\icmlaffiliation{tamu}{Department of Computer Science \& Engineering, Texas A\&M University, TX, USA}
\icmlaffiliation{uiowa}{Department of Computer Science, The University of Iowa, IA, USA}

\icmlcorrespondingauthor{Shuiwang Ji}{sji@tamu.edu}

\icmlkeywords{Machine Learning, ICML}

\vskip 0.3in
]



\printAffiliationsAndNotice{\icmlEqualContribution} 

\begin{abstract}
Training of graph neural networks (GNNs) for large-scale node classification is challenging. A key difficulty lies in obtaining accurate hidden node representations while avoiding the neighborhood explosion problem. Here, we propose a new technique, named feature momentum (FM), that uses a momentum step to incorporate historical embeddings when updating feature representations. We develop two specific algorithms, known as GraphFM-IB and GraphFM-OB, that consider in-batch and out-of-batch data, respectively.
GraphFM-IB applies FM to in-batch sampled data, while GraphFM-OB applies FM to out-of-batch data that are 1-hop neighborhood of in-batch data.
We provide a convergence analysis for GraphFM-IB and some theoretical insight for GraphFM-OB. Empirically, we observe that GraphFM-IB can effectively alleviate the neighborhood explosion problem of existing methods. In addition, GraphFM-OB achieves promising performance on multiple large-scale graph datasets.

\end{abstract}

\section{Introduction}
Graph neural networks (GNNs) achieve promising performance on many graph learning tasks, such as node classification~\cite{kipf2017semi, hamilton2017inductive, velivckovic2018graph,  liu2020towards}, graph classification~\cite{xu2018powerful,gao2019graph,gao2021topology}, link prediction~\cite{zhang2018link,CaiMlinkAAAI20}, and molecular property prediction~\cite{gilmer2017neural,liu2022spherical}. In general, GNNs learn from graphs by the popular message passing framework~\cite{gilmer2017neural}. Specifically, we usually perform a recursive aggregation scheme in which each node aggregates representations from all 1-hop neighbors. Various GNNs~\citep{kipf2017semi,velivckovic2018graph,xu2018powerful,li2019deepgcns} mainly differ in the employed aggregation functions. Such recursive aggregation scheme has been shown to be effective for learning graph representations. However, it leads to the inherent neighborhood explosion problem~\cite{hamilton2017inductive}, since the number of neighbors grows exponentially with the depth of GNNs.


Due to such inherent problem, we have difficulties to apply GNNs to large-scale graphs. Notably, many real-world graphs, such as citation networks, social networks, and co-purchasing networks, are large-scale graphs \cite{hu2021ogblsc} with massive numbers of nodes and edges. Thus, it is hard to obtain the complete computational graph, which contains the exploded neighborhood, with limited GPU memory when training on large-scale graphs. 
To tackle this,
efficient training algorithms~\cite{huang2018adapt, gao2018large, bojchevski2020scaling, cwdlydw2020gbp, huang2020combining, you2020l2, li2021training, wan2022pipegcn, liu2022exact} have been developed to update model parameters with reasonable computational complexity and memory consumption. They aim to obtain accurate hidden node representations and gradient estimations while avoiding the neighborhood explosion problem.
The mainstream efficient training algorithms are sampling methods. They perform node-wise, layer-wise, or graph sampling to alleviate the neighborhood explosion problem.
However, sampling incurs unavoidable errors in estimation on hidden node embedding and gradients. In section \ref{Method}, we formulate GNNs as recursive nonlinear functions and show that the gradients of the sampling methods suffer from estimation error due to the nonlinearity. In addition, establishing a convergence guarantee of Adam-style algorithms is another challenge for sampling based methods~\cite{cong2020importance}.

In this work, we propose a novel technique, known as feature momentum (FM), to address these problems.
FM applies a momentum step on historical node embeddings for estimating accurate hidden node representations. Based on FM, we develop two algorithms based on sub-sampled node estimation and pseudo-full neighborhood estimation, respectively. The first algorithm, known as GraphFM-IB, applies FM after sampling the in-batch nodes. GraphFM-IB samples the 1-hop neighbors of target nodes recursively and then updates historical embeddings of them with aggregated embeddings from their sampled neighbors using a momentum step. 
The second algorithm, known as GraphFM-OB, uses cluster-based sampling to draw in-batch nodes and employs the FM to update historical embeddings of 1-hop out-of-batch nodes with the message passing from in-batch nodes. The operations of GraphFM-OB are given in Figure \ref{fig:graphfm_ob_comp}. 
GraphFM-IB was shown theoretically to converge to a stationary point with enough iterations and a constant batch size. GraphFM-OB is shown to have some theoretical insight of possibly alleviating the staleness problem of historical embeddings. 

We perform extensive experiments to evaluate our methods on multiple large-scale graphs. Results show that GraphFM-IB outperforms GraphSAGE and achieves comparable results with other baselines. Importantly, when sampling only one neighbor, GraphFM-IB achieves similar performance as when large batch sizes are used, thus alleviating the neighborhood explosion problem. We also show that GraphFM-OB outperforms current baselines on various large-scale graph datasets.

\section{Related Work}
GNNs are powerful methods for learning graph representations \cite{gori2005new, scarselli2008graph}. Commonly used GNNs include GCN \cite{kipf2017semi}, GCNII \cite{chen2020simple}, and PNA \cite{corso2020pna}. Due to the neighborhood explosion problem, full-batch training of GNNs on large-scale graphs incurs prohibitive GPU memory consumption. Therefore, it is practically desirable to develop efficient training strategies on large-scale graphs. Recently, several categories of sampling methods have been proposed to reduce the size of the computational graph during training, including node-wise, layer-wise, and graph sampling methods.


\textbf{Node-wise sampling methods} uniformly sample a fixed number of neighbors when performing recursive aggregation and usually learn model parameters using gradients on a batch of nodes instead of all nodes. Such training strategy is originally proposed in GraphSAGE \cite{hamilton2017inductive}.
Then VR-GCN \cite{chen2018stochastic} integrates historical embeddings with GraphSAGE to reduce the estimation variance. Moreover, it provides a convergence analysis, which demonstrates that VR-GCN can converge to a local optimum with infinite iterations.
However, the analysis assumes the unbiasedness of stochastic gradients, which is usually unrealistic. 
In addition, although node-wise sampling methods reduce memory requirement to some extend, they still suffer from the neighborhood explosion problem.

\textbf{Layer-wise sampling methods} select a fixed number of nodes in each layer according to their defined sampling distribution. They overcome the neighborhood explosion problem because the number of neighbors grows only linearly with depths.
FastGCN \cite{chen2018fastgcn} performs such sampling independently in each layer. In contrast,
LADIES \cite{zou2019layer} moves one step forward to consider the dependency of sampled nodes between layers. 
They both provide the unbiasedness and variance analysis of node embedding for one-layer GNNs without nonlinear function. MVS-GCN \cite{cong2020minimal} further analyzes multi-layer GNNs with nonlinearity. It formulates the training of GNNs as a stochastic compositional optimization problem. Since the stochastic estimator of the gradient is biased, a large batch size is required to eliminate the bias and variance for the convergence guarantee. Layer-wise sampling methods can effectively alleviate the memory bottleneck. Nonetheless, layer-wise sampling methods have incurred computational overhead since we have to perform extensive sampling during training. 

\textbf{Graph sampling methods} sample subgraphs or clusters to construct minibatches before training. Note that we only need to obtain such subgraphs or clusters once as a preprocessing step. Specifically, GraphSAINT \cite{zeng2019graphsaint} samples nodes, edges or random walks to construct subgraphs. 
SHADOW \cite{zeng2021decoupling} samples subgraphs according to PageRank scores of local neighbors for each node. It builds subgraphs for each node and converts a node classification task into a graph classification task.
ClusterGCN \cite{chiang2019cluster} aims to reduce the effect of cutting graphs by spliting a graph into separate clusters with clustering algorithms.
GNNAutoScale \cite{Fey/etal/2021} proposes to incorporate historical embeddings of the dropped edges among the clusters, thereby obtaining more accurate full-batch neighborhood estimation.

Another orthogonal direction for large-scale graph training is precomputing methods \cite{wu2019simplifying, frasca2020sign, liu2022neighbor2seq}. They aggregate the multi-hop features for each node on the raw input features by precomputing and then feed them into subsequent models. Since the precomputing procedure does not involve any learnable parameters, the training of each node is independent. They are efficient but do not employ powerful GNNs with nonlinearity.

\section{The Proposed Feature Momentum Method} \label{Method}

In this section, we present the proposed methods. We first introduce some notations, then present the motivation of the proposed algorithmic design, and then discuss two methods for improving GNN training. 

{\bf Notations.} Let $\mathcal V=\{1, \ldots, n\}$ denote a set of nodes with input features denoted by $\x_v\in\R^{d_0}$ for node $v$. For any node $v$, we denote by $\N_v$ the neighboring nodes of $v$ that have connections with $v$. Let $h^k_v\in\R^{d_k}$ denote the feature vector of the $v$-th node at the $k$-th layer, where $h^0_v=\x_v$ and $h^k_v, \forall k=1, \ldots, K$ are recursively computed.  Let $W^k$ denote the model parameters at the $k$-th layer. For simplicity, we consider supervised learning tasks, where at the output layer we optimize the following loss:
\begin{align}
    \min_{W_1, \ldots, W_{K+1}}F(W)=\sum_{v\in\S}\ell(W; h^K_v)
\end{align}
where $\S\subset\mathcal V$ denotes a subset of nodes that of interest for supervised learning.  Let $H^k=[h^k_1, \ldots, h^k_n]\in\R^{n\times d_k}$ denote concatenated representations of all nodes at $k$-th layer. We can write $F(W)=f_{K+1}(H^K)$ for some non-linear function $f_{K+1}$.

{\bf Motivation of Algorithmic Design.} The feature representations of all nodes at each layer are recursively computed. We can express $H^{k}$ as 
\begin{align*}
    H^{k} = f_{k}(H^{k-1}),
\end{align*}
where $f_{k}$ is a parameterized non-linear function. As a result, we can write the optimization problem as a multi-level nested function: 
\begin{align}\label{eqn:gnn}
    F(W) = f_{K+1}\circ f_K\circ\ldots\circ f_1(H^0).
\end{align}
The two challenges in solving GNN optimization problem are (i) the representation of a node at a higher layer $h_v^k$ might depend on the representations of a large number nodes in the previous layer; (ii) not all nodes' representations can be re-computed at every iteration when the total number of nodes is large.  To this end, the GNN training is usually companioned with sampling of nodes at each layer.  However, using sub-sampled nodes to compute the feature representations at each layer might lead to a large optimization error when the sampled neighborhood of each node is not large enough. Historical embeddings of non-sampled  neighboring nodes have been used to reduce this error. However, depending on the mini-batch size, the historical embeddings could be outdated and might also have a large error. To address these issues, we introduce node-wise momentum features. To motivate this idea, we consider the following two-level functions: 
\begin{align*}
    F(\w) = f_1\circ f_2(\w). 
\end{align*}
We assume that  $f_1, f_2$ and their gradients are expensive to evaluate but unbiased  stochastic versions are readily computed. In particular, we let $f_1(\cdot; \xi)$ and $f_2(\cdot; \zeta)$ denote the  stochastic versions of $f_1$ and $f_2$ depending on random variable $\xi$ and $\zeta$, such that 
\begin{align*}
    &\E[f_1(\cdot; \xi)] = f_1(\cdot), \quad \E[\nabla f_1(\cdot; \xi)] = \nabla f_1(\cdot)\\
    &\E[f_2(\cdot; \zeta)]=f_2(\cdot), \quad \E[\nabla f_2(\cdot; \zeta)]= \nabla f_2(\cdot)
\end{align*}
The gradient of $F(\w)$ is given by $\nabla F(\w) = \nabla f_2(\w)^{\top}\nabla f_1(f_2(\w))$. An unbiased estimation is give by $\nabla\widehat F(\w)= \nabla f_2(\w; \zeta)^{\top}\nabla f_1(f_2(\w); \xi)$. It is notable that if  $f_2(\w)$ inside $f_1$ is simply replaced by $f_2(\w; \zeta)$, it will lead to a biased estimator due to non-linearity of $f_1$ and hence suffers from a large estimator error. To address this issue, existing works for two-level stochastic compositional optimization problems have proposed to use the momentum estimator (i.e., moving average)  of $f_2(\w)$ by 
\begin{align}
    \hat f_{2,t} = (1-\beta) \hat f_{2,t-1} + \beta f_2(\w_t; \zeta).
\end{align}
While this idea seems straightforward in light of multi-level stochastic compositional optimization, there are still several challenges to be addressed in the customization for GNN training: (i) the non-linear function $f_i$ in~(\ref{eqn:gnn}) does not have an unbiased estimator except for $f_{K+1}$; hence a different perspective of decomposition is required by viewing $f_{k}(H^{k-1})=\sigma (\hat f_k (H^{k-1}))$, where $\sigma$ is a simple deterministic activation function and $\hat f_k(\cdot)$ is a linear function of input; (ii) we cannot obtain unbaised estimators for all nodes in $H^{k-1}$; hence coordinate-wise sampling needs to be considered for the analysis;  (iii) if historical embeddings of out-of-batch nodes are used, which do not  give unbiased estimator for each in-batch node,  ad-hoc methods using the momentum averaging for out-of-batch nodes need to be developed; (iv) last but not least, how to provide theoretical analysis  (e.g., convergence analysis) for the proposed methods by using Adam-style update, which is mostly used for GNN training.  To the best of our knowledge, no convergence analysis of an Adam-style method has been given for multi-level compositional optimization. 

In the following two subsections, we will address these issues. We consider two  representative GNN training methods, i.e., sub-sampled neighborhood estimation and pseudo-full neighborhood estimation, where the former uses only sub-sampled neighbors for estimating the feature representations of in-batch nodes, and the latter uses all neighbors for estimating the feature representations of in-batch nodes except that for the out-of-batch neighbors historical embeddings are used. For sub-sampled neighborhood estimation, we develop a stochastic method  by using feature momentum  and provide a convergence analysis for the Adam-style method. For pseudo-full neighborhood estimation, we develop an ad-hoc method on top of a state-of-the-art method GNNAutoScale and also provide some theoretical analysis for the estimation error. 

\begin{algorithm}[t]
\caption{GraphFM-IB}\label{alg:1}
\begin{algorithmic}[1]
\REQUIRE $\eta,\{\beta_{0,k}\},\beta_1, \beta_2, \tilde h^{k,0}_v=0, \forall v\in\mathcal V$
\ENSURE $\w_T$
\FOR{$t=1,...T$}
\FOR{$k=1,\ldots, K$}
\STATE Draw a batch of nodes $\mathcal D_k$
\FOR{$v\in\mathcal D_k$}
\STATE Sample a neighborhood denoted by $\B_v$
\STATE Compute $\tilde h^{k,t}_v$ according to~(\ref{eqn:hu})
\STATE Compute $\hat h^{k,t}_v$ according to~(\ref{eqn:hf})
\STATE Normalize $\hat h^{k,t}_v$ appropriately (in practice) 
\ENDFOR
\ENDFOR
\STATE Compute the stochastic gradient estimator $G_t$ by $$G_t= \frac{1}{|\mathcal D_K|}\sum_{v\in\mathcal D_k}\nabla \ell(W_t; \hat h^{K, t}_v)$$
\STATE Compute $\mathbf v^{t+1}_1 = (1-\beta_1)\mathbf v^t_1 + \beta_1 G_t$
\STATE Compute $\mathbf v^{t+1}_2 = (1-\beta_2)\mathbf v^t_2 + \beta_2 G^2_t$
\STATE Update $W^{t+1}=W^t -\frac{\eta}{\sqrt{\mathbf v^{t+1}_2} + \epsilon_0}\mathbf v^{t+1}_1$
\ENDFOR
\end{algorithmic}
\end{algorithm}

\begin{algorithm}[t]
\caption{GraphFM-OB}\label{alg:2}
\begin{algorithmic}[1]
\REQUIRE $\eta,\{\beta_{0,k}\},\beta_1, \beta_2, \tilde h^{k,0}_v=0, \forall v\in\mathcal V$
\ENSURE $\w_T$
\FOR{$t=1,...T$}
\STATE Draw a batch of nodes $\mathcal D_t$
\FOR{$k=1,\ldots, K$}
\STATE Compute $\tilde h_v^{k,t}$ for $v\not\in\mathcal D_t$ according to~(\ref{eqn:th})
\FOR{$v\in\mathcal D_t$}
\STATE Compute $h^{k,t}_v$ according to~(\ref{eqn:fh})
\STATE Normalize $h^{k,t}_v$ appropriately (in practice) 
\ENDFOR
\ENDFOR
\STATE Compute the stochastic gradient estimator $G_t$ by $$G_t= \frac{1}{|\mathcal D_t|}\sum_{v\in\mathcal D_t}\nabla \ell(W_t; h^{K, t}_v)$$
\STATE Compute $\mathbf v^{t+1}_1 = (1-\beta_1)\mathbf v^t_1 + \beta_1 G_t$
\STATE Compute $\mathbf v^{t+1}_2 = (1-\beta_2)\mathbf v^t_2 + \beta_2 G^2_t$
\STATE Update $W^{t+1}=W^t -\frac{\eta}{\sqrt{\mathbf v^{t+1}_2} + \epsilon_0}\mathbf v^{t+1}_1$
\ENDFOR
\end{algorithmic}
\end{algorithm}

\subsection{Feature Momentum for In-Batch Nodes}
 Let us consider the computation of one node's representation.  A basic operator in GNN training is to compute new feature embeddings  of each node from its neighborhood nodes using their lower level feature embeddings. This operator can be expressed by the following two steps: 
\begin{align}
    &h^k_{\mathcal N_v} = \text{Aggregate}_k(\{h_u^{k-1}, \forall u\in\N(v)\})\\
    &h^k_{v} = \sigma(W^k\cdot\text{Concat}(h^{k-1}_v, h^k_{\mathcal N_v})),
\end{align}
where the first step aggregates the representations of the nodes in the immediate neighborhood of node $v$ into a single vector $h^k_{\N(v)}$, and the second step  concatenates the node’s current representation $h^{k-1}_v$, with the aggregated neighborhood vector and passes it through a non-linear layer with an activation function $\sigma(\cdot)$ and weights $W_k$. Of particular interest, we consider the mean aggregator.
\begin{align*}
 &\mathcal A(\{h^{k-1}_u: u\in \mathcal N_v\cup \{v\}\})=\frac{1}{|\N_v|+1}\sum_{u\in\N(v)\cup\{v\}}h^{k-1}_u\\
  &  h^k_v = \sigma\left(W^k\cdot \mathcal A(\{h^{k-1}_u: u\in \mathcal N_v\cup \{v\}\}) \right),
\end{align*}
where $\mathcal A(\cdot)$ denotes the mean operator.  To tackle the first challenge that involves a large neighborhood size, we use feature momentum with stochastic sampling to estimate the aggregated feature vector (before taking sigmoid) at the $t$-th iteration. To this end, we let $\B_v\subset\N_v$ denote a sub-sampled neighborhood of node $v$, and let $\bar\B_v=\B_v\cup\{v\}$. Then, we estimate the aggregated feature vector by 
\begin{align}\label{eqn:FM}
 \resizebox{0.9\hsize}{!}{$
    \tilde h_v^{k, t}= (1-\beta_{0,k})\tilde h^{k, t-1}_v+ \beta_{0,k}\hat\A(\{\hat h^{k-1, t}_u, u\in\bar\B_v\}),$}
\end{align}
where $\beta_{0,k}\in(0,1)$ is a momentum parameter, and $\hat\A(\{h^{k-1}_u, u\in\bar\B_v\})$ denotes an unbiased estimator of $\mathcal A(\{h^{k-1}_u: u\in \mathcal N_v\cup \{v\}\})$. We can compute it by 
\begin{align}
    \hat\A(\{\hat h^{k-1, t}_u, u\in\bar\B_v\})&=\frac{1}{|\N_v+1|}\hat h^{k-1, t}_v\\
    &+ \frac{|\N_v|}{|\N_v|+1}\frac{1}{|\B_v|}\sum_{u\in\B_v}\hat h^{k-1, t}_u.\notag
\end{align}
To tackle the second challenge that computing $\tilde h_v^k$ for all nodes $v\in\mathcal V$ is expensive, we only compute it for a sub-sampled set of nodes denoted by $\mathcal D_k$ (in-batch nodes), i.e., 
\begin{align}\label{eqn:hu}
    &\tilde h_v^{k, t}=\\
    &\left\{\begin{array}{ll}\hspace*{-0.05in}\tilde h^{k, t-1}_v &\hspace*{-0.8in} \text{ if } v\not\in\mathcal D_k\\\hspace*{-0.05in} (1-\beta_{0, k})\tilde h^{k, t-1}_v+ \beta_{0,k} \hat\A(\{\hat h^{k-1, t}_u, u\in\bar\B_v\})& \hspace*{-0.1in}\text{ o.w.}\end{array}\right.\nonumber
\end{align}
With these momentum features, we can update the next layer feature by 
\begin{align}\label{eqn:hf}
    \hat h^{k,t}_v = \sigma(W^t_k\cdot \tilde h_v^{k, t}).  
\end{align}
The above procedure will be repeated for $K$ times for computing the output feature representations $\hat h^{K,t}_v$ for sub-sampled $v\in\mathcal D_{K}$. 

We present the detailed steps of the proposed method based on sub-sampled neighborhood estimation in Algorithm~\ref{alg:1}, to which we refer as GraphFM-IB. The model parameter $\w_{t+1}$ is updated by the Adam-style update.

\begin{figure*}[t]
    \centering
    \includegraphics[width=2\columnwidth]{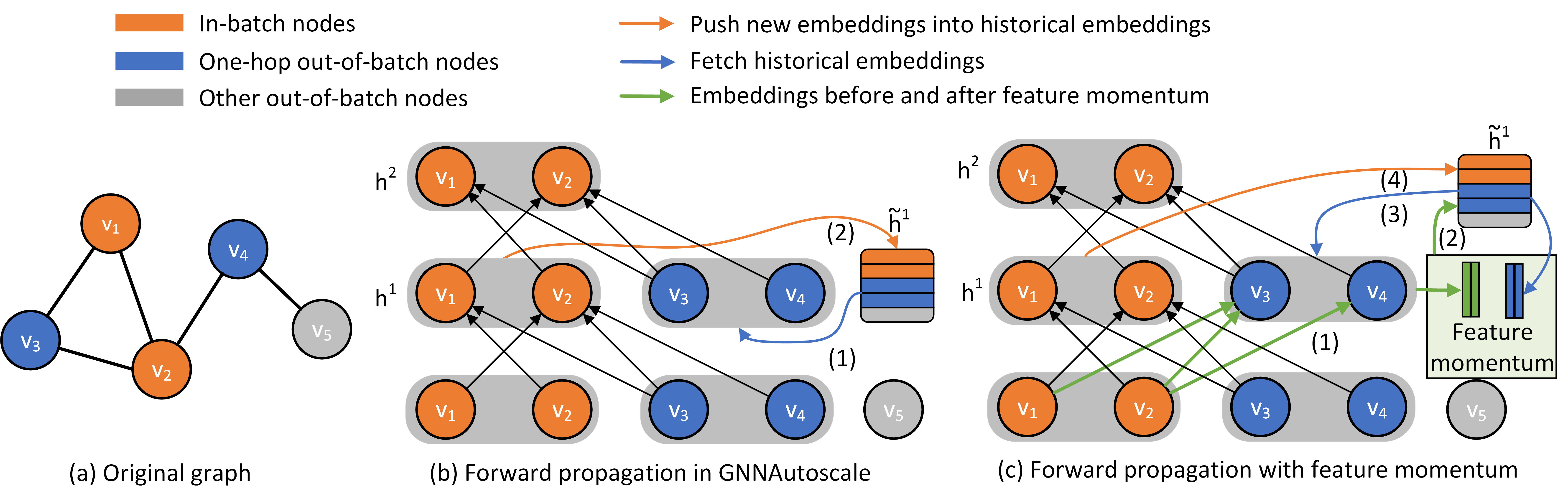}
    \caption{Comparsion of GrpahFM-OB with GNNAutoScale. (a) shows the original graph with in-batch nodes, one-hop out-of-batch nodes in blue and orange, respectively. (b) denotes the forward propagation in GNNAutoScale composed of two steps. The fist step fetches the historical embeddings for the one-hop out-of-batch nodes. Then it saves the in-batch nodes activation to their historical embeddings. Next, pseudo-full neighborhood propagation can be done to estimate the node embeddings of in-batch nodes in the next layer. (c) is the forward procedures in GraphFM-OB. It contains four steps. The first step is to calculate the message passing from the in-batch nodes to the out-batch nodes. Then we apply feature momentum to update the historical embeddings of the one-hop out-of-batch nodes, and save them into the corresponding historical embeddings. The third and last step is the same as the forward procedure in GNNAutoScale.}
    \label{fig:graphfm_ob_comp}
\end{figure*}

\paragraph{Convergence Analysis.} Next, we provide convergence analysis of GraphFM-IB. We show that GraphFM-IB converges to a stationary solution after a large number of iterations without using a large neighborhood size, which can effectively avoid  the neighbor explosion issue of existing node-wise sampling methods for large-scale GNN training. The detailed proof is provided in Appendix \ref{sec:graphfm_ib_provable}.

\begin{thm}\label{thm:main}
Under proper conditions, with $\eta=O(\epsilon^K), \beta_1=O(\epsilon^K), 0<\beta_2<1, \beta_{0,k}=O(\epsilon^{K-k}), T=O(\epsilon^{-(K+2)})$, GraphFM-IB ensures to find an $\epsilon$-stationary solution such that $\E[\|\nabla F(\w_\tau)\|]\leq \epsilon$ for a randomly selected $\tau\in\{1,\ldots, T\}$.  
\end{thm}

\subsection{Feature Momentum for Out-of-Batch Nodes}

\begin{table*}[t]
\caption{{Statistics and properties of the datasets. The ``m'' denotes the multi-label classification task, and ``s'' denotes single label classification task.}}
\vspace{-5pt}
\addtolength{\tabcolsep}{-4pt}
\centering
\begin{tabular}{l|cccccc}
\toprule
Dataset         & \# of nodes   &  \# of edges  & Avg. degree  & \# of features  & \# of classes & Train/Val/Test    \\
\midrule 
Flickr          & 89,250      &    899,756    & 10.0813    & 500  &  7(s) & 0.500/0.250/0.250   \\
Yelp            & 716,847     &   6,997,410   &  9.7614    & 300  & 50(m) & 0.750/0.150/0.100  \\
Reddit          & 232,965     &   11,606,919  & 49.8226    & 602  & 41(s) & 0.660/0.100/0.240   \\
ogbn-arxiv      & 169,343     &   1,166,243   & 6.8869     & 128  & 40(s) & 0.537/0.176/0.287\\
ogbn-products   & 2,449,029   &   61,859,140  & 25.2586    & 100  & 47(s) & 0.100/0.020/0.880 \\
\midrule
\end{tabular}
\label{table:dataset statistics}
\end{table*}

\begin{figure*}[t]
     \centering
     \subfloat[Comparison between GraphSAGE and graphFM-IB + SAGE on Reddit.]
     {\includegraphics[width=0.3\textwidth]{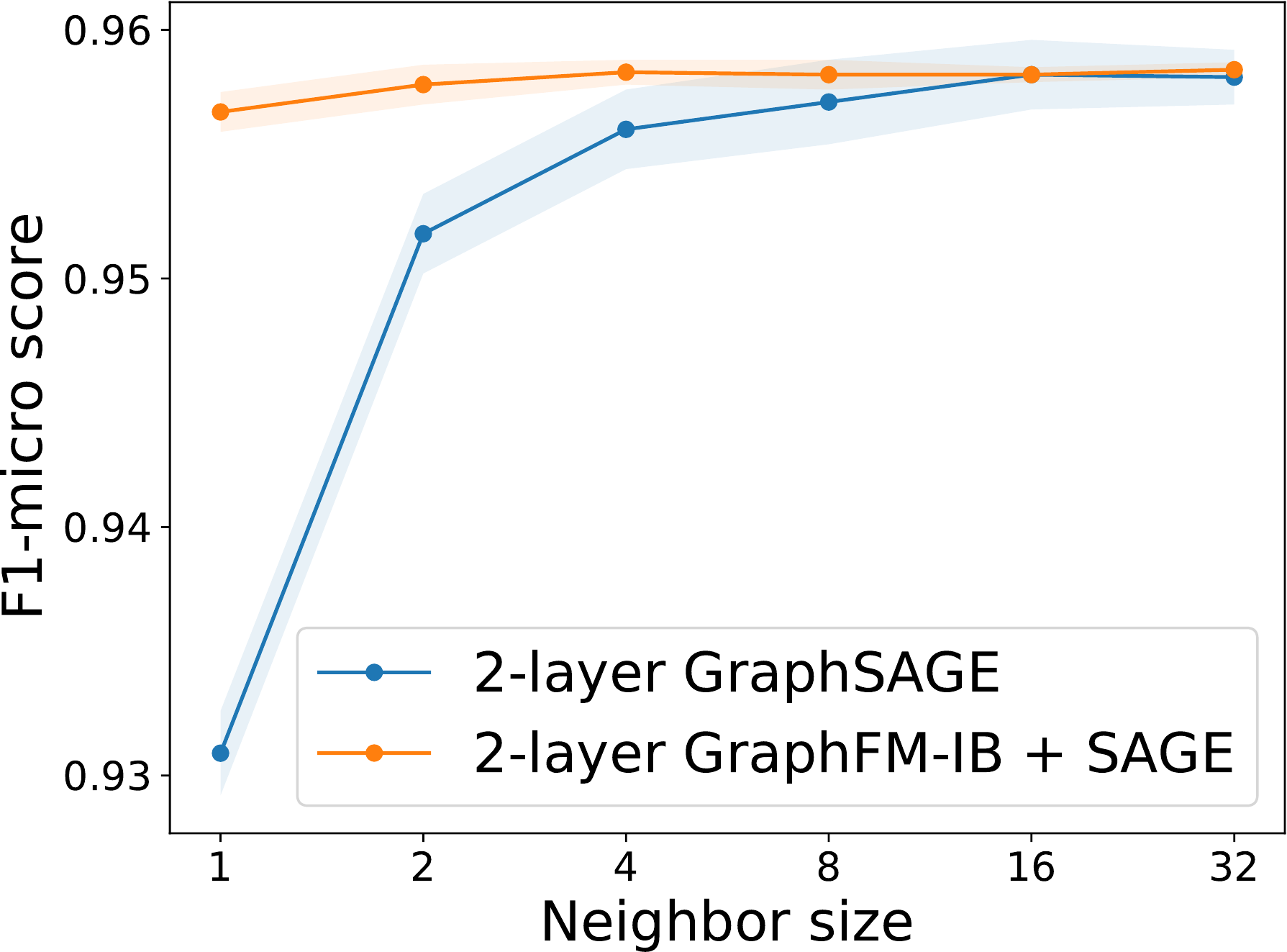}\label{fig:reddit}}
     \qquad
     \subfloat[Comparison between GraphSAGE and graphFM-IB + SAGE on Flickr.]
     {\includegraphics[width=0.3\textwidth]{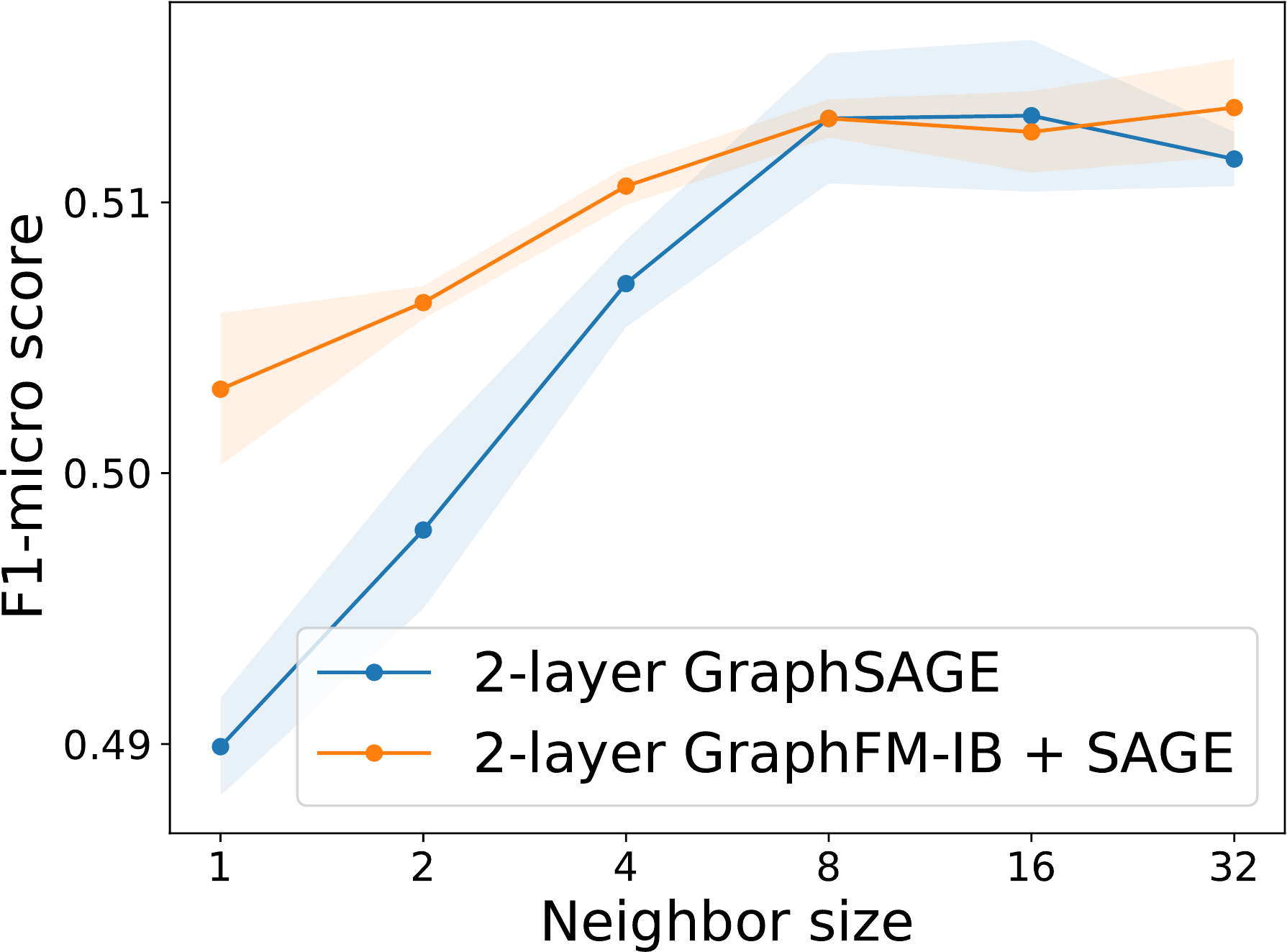}\label{fig:flickr}}
     \qquad
     \subfloat[Results for graphFM-IB + SAGE with different number of layers and neighbor sizes on Reddit.]
     {\includegraphics[width=0.305\textwidth]{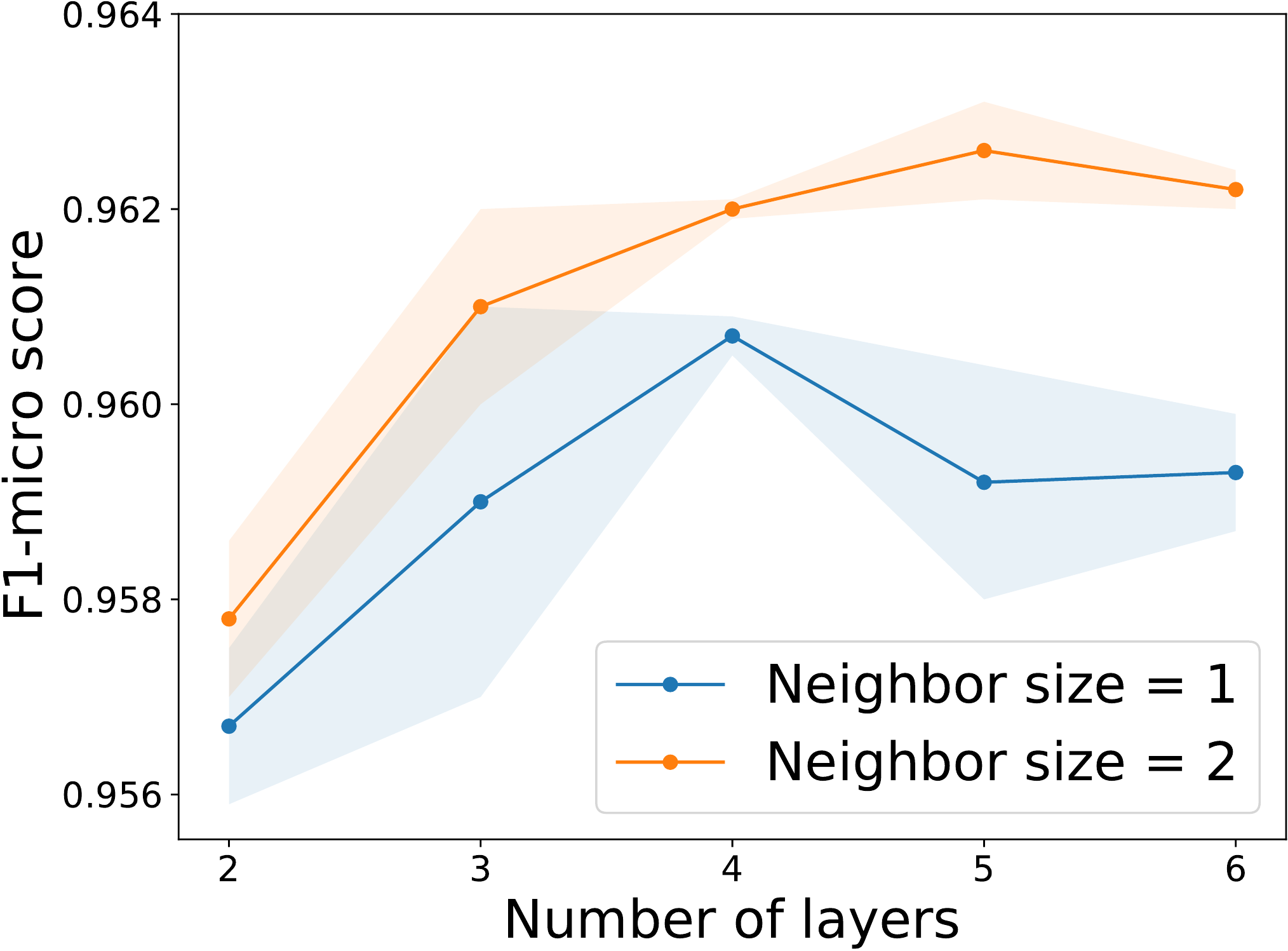}\label{fig:reddit_layer}}
    \caption{Illustration of the difference between GraphSAGE and graphFM-IB + SAGE and the performance of graphFM-IB + SAGE with different number of layers. 
    Neighbor size denotes the sampled neighbor size for each node at every layer.}
    \label{fig:ib}
\end{figure*}

In this subsection, we present a method based on GNNAutoScale by applying Feature Momentum to out-of-batch nodes. To this end, we first describe GNNAutoScale~\cite{Fey/etal/2021}.  It is based on pseudo-full neighborhood estimation that uses all neighbors for computing the new feature representation of an in-batch node. However, not all neighbors have updated their feature representations in the previous layer due to that some are not in the sampled batch. To address this issue, historical embeddings of those out-of-batch nodes are used. Let $\mathcal D_t$ denote the in-batch nodes. Then for each node $v\in\mathcal D_t$, we denote its neighborhood  including itself by $\overline{\mathcal N_v}=\mathcal N_v\cup \{v\}$. We can decompose it into two subsets, $\mathcal S^t_v=\overline{\mathcal N_v}\cap \mathcal D_t$ and $\mathcal O^t_v=\overline{\mathcal N_v}\setminus \mathcal S^t_v$. In GNNAutoScale, the $k$-th layer embedding of node $v\in\mathcal D_t$ is computed as:
\begin{align}\nonumber
    &\hat h^{k,t}_v = W_k^t\cdot\mathcal A(\{h^{k-1,t}_u: u\in\mathcal S^t_v\}\cup\{\tilde{h}^{k-1, t}_u: u\in\mathcal O^t_v\})\\\label{eqn:gnnas}
    &h^{k,t}_v=\sigma(\hat h^{k,t}_v)
\end{align}
where $\tilde{h}^{k-1,t}_u$ denotes the fetched historical embedding of the out-batch node $u\in\O_v^t$. We denote by $\tau_u$ ($\tau_u<t$)  the last iteration before $t$ that $u$ is sampled, i.e., $u\in\mathcal D_{\tau_u}$. Due to the update rule of GNNAutoscale, it is worth noting that the historical embedding $\tilde{h}_u^{k-1}$ of $u$ in layer $k-1$ is not updated between iteration $\tau_u$ and $t$, i.e., $\tilde{h}_u^{k-1,t}= \dotsc = \tilde{h}_u^{k-1,\tau_u} = h_u^{k-1,\tau_u}$. Depending on the batch size and the size of the graph, the last updated feature representation of node $u$ could happen a large number of iterations ago $\tau_u\ll t$, hence $\tilde{h}_u^{k-1,t} = h^{k-1,\tau_u}_u$ might have a large estimation error compared to $h_u^{k-1,t}$. In practice, cluster-based sampling has been shown to be helpful to reduce $t-\tau_u$~\citep{Fey/etal/2021}.  
	
	To further mitigate this issue, we notice that for those in out-of-batch (at layer $k$) such that their 1-hop neighbors are in the sampled batch (at layer $k-1$), we can use their in-batch neighbors to update their historical embeddings using moving average similar to that in~(\ref{eqn:FM}). 
	Then we update the historical embedding $\tilde{h}^{k, t}_v$ for $v\notin\mathcal D_t$ by 
	\begin{equation}\label{eqn:th}
		\begin{aligned}
			\tilde{h}^{k,t}_v &=(1-\beta_{0,k})\tilde{h}^{k, t-1}_v \\
			&\quad+ \beta_{0,k} \sigma(W_k^t\cdot \A(\{h_u^{k-1,t}:u\in\S_v^t\})).
		\end{aligned}
	\end{equation}
	Indeed, we do not need to update $\tilde{h}^{k,t}_v$ for all $v\not\in\mathcal D_t$. We only need to consider those such that $\S_v^t\neq\emptyset$, $v\not\in\mathcal D_t$.
	
	Then, we can update $h^{k,t}_v$ for $v\in\mathcal D_t$ by 
	\begin{equation}\label{eqn:fh}
		\begin{aligned}
			&\hat h^{k,t}_v = \mathcal A(\{h^{k-1,t}_u: u\in\mathcal S^t_v\}\cup\{\tilde h^{k-1, t}_u: u\in\mathcal O^t_v\})\\
			&h^{k,t}_v=\sigma(W_k^t\cdot \hat h^{k,t}_v).
		\end{aligned}
	\end{equation}
	We present a formal description of the proposed method GraphFM-OB in Algorithm~\ref{alg:2}. A computational graph of GraphFM-OB compred with GNNAutoScale is shown in Figure~\ref{fig:graphfm_ob_comp}.
	Some theoretical insight showing that the updated $h^{k,t}_v$ in~(\ref{eqn:fh}) by GraphFM-OB could lead to smaller estimation error than the updated $h^{k,t}_v$ in~(\ref{eqn:gnnas}) by GNNAutoScale can be found in Appendix~\ref{sec:graphfm_ob_insight}. 


\section{Experiments}

\begin{table*}[t]
\caption{Comparison between our GraphFM-IB, GraphFM-OB and other baseline methods. The reported results of GraphFM-IB have averaged over 5 random runs. The experiments of GraphFM-OB follow the setting of GNNAutoScale to report the F1-micro scores with fixed random seeds for a fair comparison. The top performance scores are highlighted in \textbf{bold}. \underline{Underline} indicates that our methods achieve better performance compared to the corresponding baselines without feature momentum.}
\centering
\begin{tabular}{ll|ccccc}
\toprule 
Backbones & Methods  & Flickr   & Reddit  & Yelp  & ogbn-arxiv   & ogbn-products   \\
\midrule 
  & VR-GCN         & 0.482 ± 0.003    &   0.964 ± 0.001      & 0.640 ± 0.002    & --      &  --           \\
  & FastGCN        & 0.504 ± 0.001    &   0.924 ± 0.001      & 0.265 ± 0.053    & --      &  --          \\
  & GraphSAINT     & 0.511 ± 0.001    &   0.966 ± 0.001      & 0.653 ± 0.003    & --      &  0.791 ± 0.002   \\
  & Cluster-GCN    & 0.481 ± 0.005    &   0.954 ± 0.001      & 0.609 ± 0.005    & --      &  0.790 ± 0.003   \\
  & SIGN           & 0.514 ± 0.001    &   0.968 ± 0.000      & 0.631 ± 0.003    & 0.720 ± 0.001      &  0.776 ± 0.001   \\
\midrule
 \multirow{2}{*}{SAGE}  & GraphSAGE      & 0.501 ± 0.013    &   0.953 ± 0.001      & 0.634 ± 0.006    & 0.715 ± 0.003      & 0.783 ± 0.002 \\
 & GraphFM-IB  & 0\underline{.513 ± 0.009}    & \underline{0.963 ± 0.005}        & \underline{0.641 ± 0.001}    & 0.713 ± 0.002      & \underline{0.792 ± 0.003} \\
 \midrule
 \multirow{2}{*}{GCN} & GNNAutoScale         & 0.5400     & 0.9545      &  0.6294      & 0.7168    & 0.7666       \\
 & GraphFM-OB   & \underline{0.5446}	 &  0.9540	   &   --         & \underline{0.7181}	   & \underline{0.7688}      \\
  \midrule
 \multirow{2}{*}{GCNII} & GNNAutoScale       & 0.5620     & 0.9677      & 0.6514        & 0.7300     & 0.7724  \\
 & GraphFM-OB      & \underline{0.5631}     & \underline{0.9680}      & \textbf{\underline{0.6529}}       & \textbf{\underline{0.7310}}     & \underline{0.7742}         \\
  \midrule
 \multirow{2}{*}{PNA} & GNNAutoScale         & 0.5667     & \textbf{0.9717}      & 0.6440        & 0.7250     & 0.7991    \\
  & GraphFM-OB   & \textbf{\underline{0.5710}} & 0.9712 & \underline{0.6450}  & \underline{0.7290}   & \textbf{\underline{0.8047}}         \\
\bottomrule
\end{tabular}
\label{table:overall performance}
\end{table*}

\textbf{Datasets.}  We evaluate our proposed algorithms GraphFM-IB and GraphFM-OB with extensive experiments on the node classification task on five large-scale graphs, including Flickr \cite{zeng2019graphsaint}, Yelp \cite{zeng2019graphsaint}, Reddit \cite{hamilton2017inductive}, ogbn-arxiv \cite{hu2021ogblsc} and ogbn-products \cite{hu2021ogblsc}. They contains thousands or millions of nodes and edges, and we summarize the statistics of these datasets in Table \ref{table:dataset statistics}. The detailed task description of these datasets can be found in Appendix \ref{App:datasets}.

\textbf{Baselines.} We compare with the following five baselines: 1. VR-GCN \cite{chen2018stochastic}, 2. FastGCN \cite{chen2018fastgcn}, 3. GraphSAINT \cite{zeng2019graphsaint}, 4. ClusterGCN \cite{chiang2019cluster}, 5. SIGN \cite{frasca2020sign}. They cover different categories of efficient algorithms on large-scale graph training, including node-wise, layer-wise, graph sampling and precomptuing methods. Comparison and details about these baselines are provided in Appendix \ref{App:baselines}.

\textbf{Software and Hardware.} The implementation of our methods is based on the PyTorch \cite{paszke2019pytorch}, and Pytorch\_geometric \cite{Fey/Lenssen/2019}. Our code is implemented in the DIG (Dive into Graphs) library \cite{JMLR:v22:21-0343}, which is a turnkey library for graph deep learning research and publicly available\footnote{\url{https://github.com/divelab/DIG/tree/dig/dig/lsgraph}}. In addition, we conduct our experiments on Nvidia GeForce RTX 2080 with 11GB memory, and Intel Xeon Gold 6248 CPU. 

\subsection{Feature Momentum for In-Batch Nodes}


\begin{table*}[t]
    \begin{center}
        \caption{Comparison between with and without feature momentum of GraphSAGE in terms of \textbf{model performance, GPU memory consumption and running time per epoch} on Reddit and Flickr. Neighbor sizes are the list of number of neighbors sampled in each layer and the authors of GraphSAGE use the sampled neighbor sizes as 25 and 10.}
    \label{table:memory_time}
    \begin{tabular}{l|l|cc}\toprule
            Methods & Neighbor sizes  & Reddit & Flickr \\
            \midrule
            GraphSAGE & 2 layer full-batch  & OOM & 0.513/4,860M/1.7s \\
            GraphSAGE & [25,10]  & 0.957/3,080M/6.5s & 0.512/1,740M/1.6s \\
            GraphSAGE & [1,1]  & 0.931/2,250M/3.3s & 0.490/1,310M/1.2s \\
            GraphFM-IB + SAGE  & [1,1] & 0.957/2,300M/3.9s & 0.503/1,480M/1.4s  \\
            GraphSAGE & [4,4] & 0.955/2,320M/4.0s & 0.507/1,390M/1.3s  \\
            GraphFM-IB + SAGE  & [4,4] & 0.958/2,450M/4.2s & 0.511/1,540M/1.5s  \\
            \midrule
            GraphSAGE & 4 layer full-batch & OOM & 0.514/11,000M/5.2s \\
            GraphSAGE & [25,10,10,10] & 0.962/10,110M/53s & 0.514/6,480M/3.6s \\
            GraphSAGE & [1,1,1,1] & 0.951/2,700M/5.2s & 0.502/1,360M/1.7s \\
            GraphFM-IB + SAGE & [1,1,1,1]  & 0.962/2,860M/6.2s & 0.513/1,700M/2.0s \\
            GraphSAGE & [2,2,2,2]  & 0.958/2,870M/5.8s & 0.509/1,470M/1.8s \\
            GraphFM-IB + SAGE & [2,2,2,2]  & 0.963/3,130M/7.5s & 0.513/1,900M/2.4s \\
            \bottomrule
            \end{tabular}
\end{center}
\end{table*}


\textbf{Setup.} We first apply our proposed GraphFM-IB algorithm in the framework of GraphSAGE~\cite{hamilton2017inductive} and conduct experiments on five large-scale graph datasets.
GraphSAGE is designed for large-scale graph learning and provides a general framework with neighbor sampling and aggregation. 
The aggregation is implemented in a graph convolution layer, called SAGEConv.
The main idea of the neighbor sampling is to sample neighbors for each node to avoid the memory issue caused by considering full neighbors.
In practice, the authors use a fixed-size, uniform sampling function to sample neighbors for each node in each layer.
The fixed neighbor sizes can be different among layers 
and the authors use a 2-layer GraphSAGE with sampled neighbor sizes as 25 and 10.
In the following part, we use `GraphFM-IB + SAGE' to denote our GraphFM-IB method with SAGEConv.
We explore the feature momentum hyper-parameter $\beta$ in the range from 0.1 to 0.9.

\textbf{Results.}
We start by studying the effects of sampled neighbor size on the performance of GraphSAGE and GraphFM-IB + SAGE, 
and show that our proposed GraphFM-IB can achieve competitive performance with smaller sampled neighbor sizes.
As shown in  Fig.~\ref{fig:ib}\subref{fig:reddit} and Fig.~\ref{fig:ib}\subref{fig:flickr}, the performance for 2-layer GraphSAGE is highly related to the sampled neighbor size. It requires to sample at least 8 neighbors per node to guarantee good performance.
In contrast, GraphFM-IB + SAGE achieves good results with only 1 sampled neighbor.
Thus we can reduce the sampled neighbor size while obtaining competitive performance by using feature momentum.


In the next step, we explore GNNs with more layers. Since GrpahFM-IB can perform well with small sampled neighbor sizes, we focus on the cases where the neighbor sizes are 1 or 2.
As shown in Fig.~\ref{fig:ib}\subref{fig:reddit_layer}, 
models with more layers outperform the 2-layer one. 
Note that using more than 5 layers may hurt the performance since deep model may need carefully designed architecture and training strategies, which is another research problem~\citep{liu2020towards, li2019deepgcns}.


To further evaluate the memory and time efficiency of the proposed GraphFM-IB, we compare GraphSAGE and GraphFM-IB with various neighbor sizes in Table \ref{table:memory_time}.
We can observe that GraphFM-IB saves a lot of GPU memory and training time while achieving similar performance as GraphSAGE, especially when the GNNs have many layers.
For example, to achieve similar performance with four-layer SAGE on Reddit, GraphSAGE with neighbor sizes $[25, 10, 10, 10]$ costs $10,100$M and 53 seconds per epoch, while GraphFM-IB with neighbor sizes $[1, 1, 1, 1]$ costs $2,860$M and 6.2 seconds per epoch. 
Note that $[1, 1, 1, 1]$ denotes sampling one neighbor at each layer in a 4-layer GNN.
In this case, GraphFM-IB only samples one neighbor, thus alleviating the neighborhood explosion problem.
Besides, the incremental GPU memory  mainly caused by historical embeddings is acceptable. 
On Reddit, the incorporation of historical embedding costs 50M for two-layer SAGE and 160M for four-layer SAGE with neighbor sizes $[1, 1, 1, 1]$.

Finally, we explore the sampled neighbor size in the set \{1,2,4,8\} and the number of layers from 2 to 4.
The testing F1-micro scores for GraphSAGE and the model using GraphFM-IB are summarized in Table \ref{table:overall performance}. The results for GraphSAGE are taken from the referred papers and the OGB leaderboards.
The results show that using our GraphFM-IB can consistently outperform GraphSAGE on all datasets, indicating that feature momentum is helpful for better feature estimation and improves performance for large graphs. 
Note that GraphSAGE with GraphFM-IB achieves almost the same results as GraphSAGE on ogb-arxiv since the result of GraphSAGE is full-batch training with the GraphSAGE convolution layer provided by the OGB team. 
We list the number here to keep consistent with the value in OGB leaderboards and the GraphSAGE result with neighbor sizes as 25 and 10 is 0.704 ± 0.001, which is still worse than GraphFM-IB + SAGE.

\subsection{Feature Momentum for Out-Batch Nodes}

\textbf{Setup.} To evaluate our proposed GraphFM-OB algorithm, we apply it to three widely used GNN backbones, including \textbf{GCN}~\cite{kipf2017semi},~\textbf{GCNII} \cite{chen2020simple} and \textbf{PNA}~\cite{corso2020pna}. Then we conduct experiments to evaluate the obatined models on the five large-scale graphs.  
We explore the feature momentum hyper-parameter $\beta$ in the range from 0.1 to 0.9. We select the learning rate from \{0.01, 0.05, 0.001\} and dropout from  \{0.0, 0.1, 0.3, 0.5\}. Due to the over-fitting problem on the ogbn-products dataset, we set the edge drop \cite{rong2019dropedge} ratio at 0.8 during training for this particular dataset.

\textbf{Results.} 
The testing F1-micro scores are shown in Table \ref{table:overall performance}. 

It can be observed that GraphFM-OB methods with different GNN backbones outperform corresponding baselines on four datasets. In addtion, GraphFM-OB shows enhancement performance with GCNII on all five datasets. With PNA as backbone, GraphFM-OB achieves the state-of-the-art performance on the Flickr dataset. These results demonstrate that GraphFM-OB, with using feature momentum to alleviate the staleness of historical embeddings, can obtain more accurate node embedding estimation.

As an exception, the implementation of GNNAutoScale and GraphFM-OB on Yelp with GCN perform much worse than the reported score 0.6294. Therefore, we omit the this result in Table \ref{table:overall performance}. 


\textbf{Measuring the staleness of historical embeddings.}
In order to measure the staleness of the historical embeddings, we propose a new metric called staleness score. Intuitively, if the historical embeddings of nodes are bright new without any staleness, they are exact the same as the full-neighborhood propagation embeddings.
The staleness of the historical embeddings is the reason for the biasedness estimation of the pseudo full-neighborhood embeddings.
Here we define the staleness score of node $v$ at $k$-th layer as Euclidean distance of its historical embedding at layer $k$ and full-neighborhood propagation embedding at $k$-th layer. Formally,
\begin{equation}
    S^k(v) = \| \bar{h}_v^k - \tilde{h}_v^k \|,
\end{equation}
where $\bar{h}_v^k$ is the full-neighborhood propagation embedding of node $v$ at $k$-th layer.

Based on the staleness score of node $v$ at layer $k$, we further propose the staleness score for the layer $k$ averaged over all the nodes as
\begin{equation}
    S^k = \frac{1}{|V|} \sum_{v \in \mathcal{V}} S^k(v)
\end{equation}
Thus, the $S^k$ can reflect the staleness of historical embeddings of all the node at layer $k$. In this way, we can compare different methods with such metric to evaluate the staleness of their historical embeddings.

We select the PNA as the backbone with 3, 4, 3 layers on Flickr, ogbn-arxiv and Yelp, respectively. We split these datasets into 24, 40 and 40 clusters respectively, and the batch size is set to 1. For GraphFM-OB, we select feature momentum hyper-parameter $\beta$ as 0.5, 0.3 and 0.3 for these three datasets respectively. We calculate the staleness scores for the historical embeddings after the epoch that achieves the best evaluation result.

The staleness scores are shown in Table \ref{table:graphfm_ob_staleness}. We can obviously find that the GraphFM-OB achieves smaller staleness scores than GNNAutoScale in most cases, indicating that GraphFM-OB can alleviate the staleness problem.

\begin{table}
	\caption{The staleness scores of the historical embeddings}
	\label{table:graphfm_ob_staleness}
	\centering
	\begin{tabular}{lccc}\toprule
        Datasets & Layer & GNNAutoScale & GraphFM-OB \\
        \hline
        \multirow{2}{*}{Flickr} & 1 & 3.8929  &  \textbf{3.2046} \\
                                & 2 & 3.2185  &  \textbf{2.3873} \\
        \hline
        \multirow{3}{*}{ogbn-arxiv} & 1  &  8.2709 &   \textbf{5.7088} \\
                                    & 2  & 12.5646 &  \textbf{12.0062} \\
                                    & 3  &  2.0200 &   \textbf{1.4884} \\
        \hline
       \multirow{2}{*}{Yelp}        & 1  &  \textbf{3.0186} & 3.2484 \\
                                    & 2  & 4.4013 &  \textbf{3.8328}  \\
        \bottomrule
    \end{tabular}
\end{table}


\section{Conclusion}
To obtain accurate hidden node representations, we propose feature momentum (FM) to incorporate historical embeddings in an Adam-update style. Based on FM, we develop two algorithms, GraphFM-IB and GraphFM-OB, with convergence guarantee and some  theoretical insight, respectively.
Extensive experiments demonstrate that our proposed methods can effectively 
alleviate the neighborhood explosion and the staleness problems, while achieving promising results.

\section*{Acknowledgments}

This work was supported in part by National Science Foundation grant IIS-1908198
and TRIPODS grant CCF-1934904 
to Texas A\&M University, and 2110545 and Career Award 1844403 to University of Iowa.


\bibliographystyle{icml2022}

\bibliography{refs}

\clearpage

\appendix
\onecolumn

\part*{Appendix}

\section{Experiment Settings}

\subsection{Dataset Descriptions} \label{App:datasets}
We compare the results on the following five large-scale datasets. 
They cover many real world tasks, which are summarized as follows.
1. classifying the image tags with image description and edges to images with same properties. (Flickr)
2. classifying user types with their reviews and edges to their friends. (Yelp)
3. classifying the community of the posts with the posts content and edges to the posts that have been commented by the same customer. (Reddit)
4. classifying the papers with the abstract average embedding features and edges to their citation papers. (ogbn-arxiv)
5. classifying categories of products on Amazon with the product description and edges to other products that are purchased together. (ogbn-products)
We follow the official split of these datasets to conduct our experiments.

\subsection{Baseline Descriptions} \label{App:baselines}
We compare the results with five different baselines, including node-wise, layer-wise, subgraph sampling and precomputing methods. We summarize the baselines as below.
\textbf{VR-GCN} \cite{chen2018stochastic} is a node-wise sampling method. Compared to the GraphSAGE, it combines the historical embedding of one-hop neighbors and the estimated embedding of sampled nodes to reduce the variance of the unbiased estimated target node embedding.
\textbf{FastGCN} \cite{chen2018fastgcn} is a layer-wise sampling method. For each layer, it will sample fixed number of nodes from the neighborhood union of target nodes in the next layer with importance sampling to reduce the variance of unbiased estimated target node embeddings. Since it will sample fixed number of nodes in each layer, the total number of sampled nodes will grow linearly. However, for node-wise sampling method, each node will sample fixed number of neighbors from the previous layer. Thus, it will lead to an exponentially growth of sampled nodes.
\textbf{ClusterGCN} \cite{chiang2019cluster} and \textbf{GraphSAINT} \cite{zeng2019graphsaint} are subgraph-sampling methods. ClusterGCN will combine fixed number of clusters to form the subgraphs to train the model. GraphSAINT will take sampled edges or random walks to form subgraphs. Compared to the layer-wise sampling method, once the subgraph is provided, the nodes in each layer are the same as the input subgraph. Therefore, the sampling procedure can be done once before training, which can alleviate the overhead of sampling.
\textbf{SIGN} \cite{frasca2020sign} is a precomputing method. It computes the L-hop aggregated features on the raw input features first, then directly feed them into MLPs to predict the labels. Because the precomputing procedure doesn't contain any parameter, the training of each node is independent.

\section{Theoretical Insight of GraphFM-OB}\label{sec:graphfm_ob_insight}

	Consider one layer of the graph neural network. The $\underline{h_v^{k,t}}$ representation of a node $v$ in the last layer by full-neighbordhood forward propagation can be written as:
	\begin{align*}
		& \underline{h_v^{k,t}} = \sigma(\underline{\hat{h}_v^{k,t}}),\quad \underline{\hat{h}_v^{k,t}} = W_k^t\cdot \A(\{h_u^{k-1,t}:u\in\overline{\N}_v\}).
	\end{align*}
	Instead of full-neighborhood propagation on every nodes, the GNNAutoScale algorithm samples a batch of nodes $\mathcal{D}_t$ and approximate $\underline{h_v^{k,t}}$ by $h_v^{k,t}$, which is computed by incorporating the current embedding $h_u^{k-1,t}$ of in-batch node $u\in S_v^t$ and the historical embedding $\tilde{h}_u^{k-1,t}$ of out-batch node $u\in\O_v^t$ as follows
	\begin{align*}
		& h_v^{k,t} = \sigma(\hat{h}_v^{k,t}),\quad \hat{h}_v^{k,t} = W_k^t\cdot \A (\{h_u^{k-1,t}:u\in\S_v^t\}\cup \{\tilde{h}_u^{k-1,t}:u\in\O_v^t\}).
	\end{align*}
	Note that $\tilde{h}_u^{k-1,t} = h_u^{k-1,\tau_u}$, and the estimation error $\Norm{\underline{h_v^{k,t}} - h_v^{k,t}}$ depends on the ``staleness'' of the historical embedding $h_u^{k-1,\tau_u}$ compared to the real embedding $h_u^{k-1,t}$, where $u\in\O_v^t$ is an out-batch node and $\tau_u$ denotes the last iteration that node $u\in\O_v^t$ is sampled, i.e. $u\in\mathcal{D}_{\tau_u}$.  Proposition~\ref{prop:GAS} provides an upper bound of the approximation error of one node $v$ by using GNNAutoScale under some conditions.
	
	\begin{prop}\label{prop:GAS}
		Assume that the activation function $\sigma(\cdot)$ is $C_\sigma$-Lipschitz continuous and the aggregator $\A(\cdot)$ is $C_\A$-Lipschitz continuous, the weight $W_k^t$ is bounded as $\Norm{W_k^t}\leq C_k$. Then, the estimation error $\Norm{h_v^{k,t} - \underline{h_v^{k,t}}}$ of GNNAutoScale is upper bounded by its ``staleness''. Formally, 
		\begin{align*}
			\Norm{h_v^{k,t} - \underline{h_v^{k,t}}} &\leq \underbrace{C_\sigma C_k C_\A\sum_{u\in\O_v^{k,t}}\Norm{h_u^{k-1,\tau_u} - h_u^{k-1,t}}}_{\text{\ding{164}}},
		\end{align*}
		where $\tau_u$ denotes the last iteration that node $u\in\O_v^{k,t}$ is sampled, i.e. $u\in\mathcal{D}_{\tau_u}$.
	\end{prop}
	\begin{proof}
		Based on the definition, the estimation error of $h_v^{k,t}$ compared to $\underline{h_v^{2,t}}$ can be upper bounded as
		\begin{align*}
			&\Norm{h_v^{k,t} - \underline{h_v^{k,t}}} \leq C_\sigma \Norm{W_k^t\cdot \left(\A (\{h_u^{k-1,t}:u\in\S_v^t\}\cup \{h_u^{k-1,\tau_u}:u\in\O_v^t\}) - \A(\{h_u^{k-1,t}:u\in\overline{\N}_v\})\right) }\\
			& \leq C_\sigma C_k \Norm{\A (\{h_u^{k-1,t}:u\in\S_v^t\}\cup \{h_u^{k-1,\tau_u}:u\in\O_v^t\}) - \A(\{h_u^{k-1,t}:u\in\overline{\N}_v\})} \leq C_\sigma C_k C_\A\sum_{u\in\O_v^t}\Norm{h_u^{k-1,\tau_u} - h_u^{k-1,t}}.
		\end{align*}
	\end{proof}
	
	The proposition below explains why GraphFM-OB could be helpful to reduce the estimation error in some cases.
	
	\begin{prop}\label{prop:GraphFM-OB}
		Assume that the activation function $\sigma(\cdot)$ is $C_\sigma$-Lipschitz continuous and the aggregator $\A(\cdot)$ is $C_\A$-Lipschitz continuous, the weight $W_k^t$ is bounded as $\Norm{W_k^t}\leq C_k$. Then, the estimation error $\Norm{h_v^{k,t} - \underline{h_v^{k,t}}}$ of GraphFM-OB is upper bounded by
		\begin{align*}
				& \Norm{\tilde{h}_u^{k-1,t} - h_u^{k-1,t}} \\
			& \leq \underbrace{C_\sigma C_kC_\A\sum_{u\in\O_v^t} \left((1-\beta_{0,k-1})^{|\T_u^t|} \Norm{h_u^{k-1,\tau_u} - h_u^{k-1},t} + \sum_{\iota=1}^{|\T_u^t|} (1-\beta_{0,k-1})^{\iota-1}\beta_{0,k-1} \Norm{h_u^{k-1,t_u^\iota} - h_u^{k-1,t}}\right)}_{ \text{\ding{165}}} \\
			& \quad\quad\quad\quad +   \underbrace{C_\sigma^2 C_kC_{k-1}C_\A\sum_{u\in\O_v^t} \left(\sum_{\iota=1}^{|\T_u^t|} (1-\beta_{0,k-1})^{\iota-1}\beta_{0,k-1} \Norm{\A(\{h_w^{k-2,t_u^{\iota}}:w\in\S_u^{t_u^\iota}\}) - \A(\{h_w^{k-2,t_u^{\iota}}:w\in\overline{\N}_u\})}\right)}_{\text{\ding{166}}}\\
			& \quad\quad\quad\quad + \underbrace{ \beta_{0,k-1} C_\sigma^2 C_k C_{k-1}C_\A \sum_{u\in\O_v^t}\left(\Norm{\A(\{h_w^{k-2,t}:w\in\S_u^t\}) - \A(\{h_w^{k-2,t}:w\in\overline{\N}_u\})}\right)}_{\text{\ding{167}}},
		\end{align*}
	where $\tau_u$ ($\tau_u<t$) is the last iteration before $t$ that $u$ is sampled, i.e. $u\in\mathcal D_{\tau_u}$. Besides, for each $u\in\O_v^t$, there exists a sub-sequence $\T_u^t$ of $\{\tau_u,\tau_u+1, \dotsc, t\}$ satisfying $\S_u^{t_u^{\iota}}\neq \emptyset$ for every $\iota =1,\dotsc, |\T_u^t|$, where $t_u^{\iota}$ is the $\iota$-th element in $\T_u^t$.
	\end{prop}
\begin{remark}
By comparing Proposition~\ref{prop:GAS} of GNNAutoScale and Proposition~\ref{prop:GraphFM-OB} of GraphFM-OB, it can be seen that the \ding{165} term of GraphFM-OB improves upon the \ding{164} term of GNNAutoScale because $h_u^{k-1,t_u^{\iota}}$ is less out-dated w.r.t. $h_u^{k-1,t}$ compared to $h_u^{k-1,\tau_u}$. As a trade-off, GraphFM-OB  also introduces the two extra terms \ding{166} and \ding{167}, which could be controlled if $\S_u$ and $\overline{\N}_u$ has more common elements. To ensure this, one might replace the condition in line 8 of Algorithm~\ref{alg:2} by $(v\notin \mathcal D_t)\land (|\mathcal D_t\cap \overline{\N}_v|>c)$, where $c>0$ is a threshold.
\end{remark}	

	\begin{proof}
		According to the update rule of GraphFM-OB, we can derive that
		\begin{align*}
			& \Norm{h_v^{k,t} - \underline{h_v^{k,t}}} \\
			& = \Norm{\sigma\left(W_k^t\cdot \left(\A(\{h_u^{k-1,t}:u\in\S_v^t\}\cup \{\tilde{h}_u^{k-1,t}:u\in\O_v^t\})\right)\right) - \sigma(W_k^t\cdot \A(\{h_u^{k-1,t}:u\in\overline{\N}_v\}))}\\
			& \leq C_\sigma C_k \Norm{\A(\{h_u^{k-1,t}:u\in\S_v^t\}\cup \{\tilde{h}_u^{k-1,t}:u\in\O_v^t\})- \A(\{h_u^{k-1,t}:u\in\overline{\N}_v\}) } \\
			& \leq C_\sigma C_k C_\A \sum_{u\in\O_v^t}\Norm{\tilde{h}_u^{k-1,t} - h_u^{k-1,t}}.
		\end{align*}
	We denote that $\tau_u$ ($\tau_u<t$) is the last iteration before $t$ that $u$ is sampled, i.e. $u\in\mathcal D_{\tau_u}$. Besides, for each $u\in\O_v^t$, there exists a sub-sequence $\T_u^t$ of $\{\tau_u,\tau_u+1, \dotsc, t\}$ satisfying $\S_u^{t_u^{\iota}}\neq \emptyset$ for every $\iota =1,\dotsc, |\T_u^t|$, where $t_u^{\iota}$ is the $\iota$-th element in $\T_u^t$. Thus, the update rule of historical embeddings in GraphFM-OB leads to
	\begin{align*}
	& \Norm{\tilde{h}_u^{k-1,t} - h_u^{k-1,t}} \\
	& =\Norm{(1-\beta_{0,k-1})\tilde{h}_u^{k-1,t_u^{|\T_u^t|-1}} + \beta_{0,k}\sigma(W_{k-1}^t\cdot \A(\{h_w^{k-2,t}:w\in\S_u^t\})) - \sigma(W_{k-1}^t\cdot \A(\{h_w^{k-2,t}:w\in\overline{\N}_u\}))}\\
	& \leq (1-\beta_{0,k-1})^{|\T_u^t|} \Norm{h_u^{k-1,\tau_u} - h_u^{k-1},t} + \sum_{\iota=1}^{|\T_u^t|} (1-\beta_{0,k-1})^{\iota-1}\beta_{0,k-1} \Norm{h_u^{k-1,t_u^\iota} - h_u^{k-1,t}} \\
	& \quad\quad\quad\quad +   C_\sigma C_{k-1} \sum_{\iota=1}^{|\T_u^t|} (1-\beta_{0,k-1})^{\iota-1}\beta_{0,k-1} \Norm{\A(\{h_w^{k-2,t_u^{\iota}}:w\in\S_u^{t_u^\iota}\}) - \A(\{h_w^{k-2,t_u^{\iota}}:w\in\overline{\N}_u\})}\\
	& \quad\quad\quad\quad +  \beta_{0,k-1} C_\sigma C_{k-1} \Norm{\A(\{h_w^{k-2,t}:w\in\S_u^t\}) - \A(\{h_w^{k-2,t}:w\in\overline{\N}_u\})}
	\end{align*}
	\end{proof}

\section{Convergence Analysis of GraphFM-IB}
\subsection{Setting, Assumptions, and Lemmas}
 \label{sec:graphfm_ib_provable}
The task of training GNNs can be abstracted as following multi-level stochastic compositional optimization problem.
\begin{align*}
	\min_{\w\in\R^D} F(\w),\quad F(\w) = f_{K+1}\circ f_K\circ \dotsc f_k \dotsc \circ f_1(\w),
\end{align*}
where $f_1:\R^D\mapsto \R^{nd_1}$, $f_k:\R^{nd_{k-1}}\mapsto \R^{nd_k}$, $k=2,\dotsc, K$, $f_{K+1}:\R^{n d_K}\mapsto \R$. Besides, each $f_k(\cdot)\in\R^{nd_k}$, $k=1,\dotsc, K$ can be splitted into $n$ blocks of $d_k$ consecutive coordinates. We denote the $i$-th block of $f_k(\cdot)$ as $f_{k,i}(\cdot)\in\R^{d_k}$.  We make the following assumption.
\begin{ass}\label{asm:IB_lip}
	Assume that $f_{k,i}$ is $L_f$-Lipschitz continuous while $\nabla f_{k,i}$ is $L_g$-Lipschitz continuous for $k=1,\dotsc,K$. Besides, $f_{K+1}$  is $L_f$-Lipschitz continuous while $\nabla f_{K+1}$ is $L_g$-Lipschitz continuous.
\end{ass}
\begin{lemma}\cite{balasubramanian2020stochastic} Given Assumption~\ref{asm:IB_lip}, 
	F is $L_F$-smooth, where $L_F \coloneqq L_f^{2K+1}L_g\sum_{k=1}^{K+1}\frac{1}{L_f^k}$.
\end{lemma}	
The full gradient of $F(\w)$ can be computed as
\begin{align*}
	\nabla F(\w) = \nabla f_1(\w) \dotsc \nabla f_K(\y_{K-1})\nabla f_{K+1}(\y_K),\quad \y_k = f_k\circ \dotsc\circ f_1(\w).
\end{align*}
At each of the $k$-th layer ($k=1,\dotsc,K$), we sample a batch of nodes $\B_k\subseteq \{1,\dotsc, n\}$. For each node $i\in\B_k$, we approximate the function value $f_{k,i}(\cdot)$ and the Jacobian $\nabla f_{k,i}(\cdot)$ by using stochastic estimators $\hat{f}_{k,i}(\cdot)$ and $\hat{\nabla} f_{k,i}(\cdot)$ by sampling the neighborhood of node $i$. At the $(K+1)$-th layer, we approximate $\nabla f_{K+1}(\cdot)$ by $\hat{\nabla} f_{K+1}(\cdot)$.
\begin{ass}
	We assume that $\E\left[\hat{f}_{k,i}(\cdot)\right] = f_{k,i}(\cdot)$, $\E\left[\Norm{\hat{f}_{k,i}(\cdot) - f_{k,i}(\cdot)}^2\right]\leq \sigma_f^2$ and $\E\left[\hat{\nabla} f_{k,i}(\cdot)\right] = \nabla f_{k,i}(\cdot)$, $\Norm{\hat{\nabla} f_{k,i}(\cdot)}^2\leq \sigma_g^2$ for $1\leq i\leq K$. Besides, $\E\left[\hat{\nabla}f_{K+1}(\cdot)\right] = \nabla f_{K+1}(\cdot)$, $\Norm{\hat{\nabla} f_{K+1}(\cdot)}^2\leq \sigma_g^2$. 
\end{ass}
In each iteration, we update the model parameter $\w$ using the following rule.
\begin{align}\label{eq:update_u}
& \u_{k,i}^t = \begin{cases}
		(1-\beta_{0,k})\u_{k,i}^{t-1} + \beta_{0,k}\hat{f}_{k,i}(\u_{k-1,i}^{t-1}), & i\in\B_k^t\\
		\u_{k,i}^{t-1}, & i\notin \B_k^t
	\end{cases},\quad \u_k^t =\begin{bmatrix}
	\u_{k,1}^t\\\vdots\\\u_{k,n}^t
\end{bmatrix},\quad k=1,\dotsc,K, \quad \u_0^t = \w^t,\\\label{eq:update_m}
& \m^t = (1-\beta_1)\m^{t-1} + \beta_1 \left(\prod_{k=1}^K \hat{g}_k^t\right)\hat{\nabla} f_{K+1}(\u_K^t),\quad \hat{g}_k^t = \sum_{i=1}^n \frac{\mathbb{I}[i\in\B_k^t]n}{B} \hat{\nabla} f_{k,i}(\u_{k-1}^t)\mathbf{I}_{k,i},\\\label{eq:update_v}
& \v^t =   (1-\beta_2)\v^{t-1} + \beta_2 \left(\left(\prod_{k=1}^{K} \hat{g}_k^t\right)\hat{\nabla} f_{K+1}(\u_K^t)\right)^2,\\\label{eq:update_w}
& \w^{t+1} = \w^t - \eta \frac{\m^t}{\sqrt{\v^t}+\epsilon_0},
\end{align}
where $\mathbf{I}_{k,i}\in\R^{d_k\times nd_k}$ is the indicator matrix that only has one non-zero block (i.e., the $i$-th block) and $B = |\B_k^t|$. Note that only the sampled blocks in $\hat{g}_k^t$ could be non-zero while the other blocks are padded with zeros.

\begin{remark}
	When specialized to the GNN training task, the described update rule is equivalent to GraphFM-IB (Algorithm~\ref{alg:1}). The scaling factors $\frac{n^K}{B^K}$ in $\m^t$ and $\v^t$ cancel each other out if we re-define $\m^0$, $\sqrt{\v^0}$, $\epsilon_0$ to be $\frac{n^K}{B^K}$ times larger. 
\end{remark}
\begin{ass}
	There exist $c_1,c_u>0$ such that $c_l\leq \Norm{\frac{1}{\sqrt{\v_t}+\epsilon_0}}\leq c_u$.
\end{ass}
\begin{lemma}[Lemma 5 in \citealt{guo2021novel}]
	For $ \eta \leq \frac{c_l}{2c_u^2L_F}$, we have:
	\begin{align*}
		F(\w^{t+1}) \leq F(\w^t) + \frac{\eta c_u}{2}\Norm{\nabla F(\w^t) - \m^t}^2 - \frac{\eta c_l}{2}\Norm{\nabla F(\w^t)}^2 - \frac{\eta c_l}{4}\Norm{\m^t}^2.
	\end{align*}
\label{lem:starter}	
\end{lemma}	
We use $\F_t$ to denote all randomness occurred up to (include) the $t$-th iteration of any algorithm. We define $\Phi^t \coloneqq \Norm{\nabla F(\w^t) - \m^t}^2$, $\Upsilon_k^t\coloneqq \Norm{f_k(\u^{k-1}_t) - \u^t_k}^2$, $\Delta^t \coloneqq \Norm{\prod_{k=1}^{K+1} \nabla f_k(\u_{k-1}^t) - \left(\prod_{k=1}^{K} \hat{g}_k^t\right)\hat{f}_{K+1}(
\u_K^t)}^2$.
\begin{lemma}
	For $\m^t$ following \eqref{eq:update_m}, we have
	\begin{align}\label{eq:grad_recursion}
		\E\left[\Phi^{t+1}\mid \F_t\right]\leq (1-\beta_1) \Phi^t + \frac{4\eta^2c_u^2L_F^2}{\beta_1}\Norm{\m^t}^2 + 4\beta_1 K \left(\sum_{k=1}^K C_k^2\Upsilon_k^{t+1}\right) + \beta_1^2 C_\Delta,		
	\end{align}
where $C_k\coloneqq L_f^K L_g(1+L_f+\dotsc+L_f^{K-k})$, $C_\Delta \coloneqq \sum_{k=1}^{K} \frac{n^{2k}L_f^{2(K+1-k)}\sigma^{2k}}{B^k}+ \frac{n^{2K}\sigma_g^{2(K+1)}}{B^K}$
\end{lemma}
\begin{proof}
	Based on the update rule of $\m^t$, we have
	\begin{align*}
		&E\left[\Phi^{t+1}\mid \F_t\right]= \E\left[\Norm{\nabla F(\w^{t+1})-\m^{t+1}}^2\mid \F_t \right]\\
		& = \E\left[\Norm{\nabla F(\w^{t+1})-(1-\beta_1)\m^t - \beta_1  \left(\prod_{k=1}^{K} \hat{g}_k^t\right)\hat{f}_{K+1}(
			\u_K^t)}^2\mid \F_t \right]\\
		& = \E\left[ \left\|(1-\beta_1)(\nabla F(\w^t) - \m^t) + (1-\beta_1)(\nabla F(\w^{t+1}) - \nabla F(\w^t)) \right.\right.\\
		& \quad\quad\quad\quad \left.\left. +\beta_1 \left(\prod_{k=1}^{K+1} \nabla f_k(\u_{k-1}^t) - \left(\prod_{k=1}^{K} \hat{g}_k^t\right)\hat{f}_{K+1}(
		\u_K^t)\right) + \beta_1 \left(\nabla F(\w^{t+1}) -\prod_{k=1}^{K+1} \nabla f_k(\u_{k-1}^t) \right)\right\|^2\mid \F_t\right]\\
		& \leq (1-\beta_1) \Phi^t + \frac{4\eta^2c_u^2L_F^2}{\beta_1}\Norm{\m^t}^2 + 4\beta_1 \E\left[\Norm{\nabla F(\w^{t+1}) -\prod_{k=1}^{K+1} \nabla f_k(\u_{k-1}^t)}^2\mid \F_t\right] + \beta_1^2 \E\left[\Delta^{t+1}\mid \F_t\right].		
	\end{align*}

	Note that
	\begin{align*}
	\Norm{\nabla F(\w^{t+1}) - \prod_{k=1}^{K+1} \nabla f_k(\u_{k-1}^t)} \leq L_f^KL_g\sum_{k=1}^K \Norm{\y_k^{t+1} - \u_k^{t+1}},	
	\end{align*}
	where $\y_k^{t+1} \coloneqq f_k\circ f_{k-1}\dotsc f_1(\w^{t+1})$. Besides, we also have
	\begin{align*}
		\Norm{\y_k^{t+1} - \u_k^{t+1}} \leq \sum_{j=1}^k L_f^{k-j} \Norm{f_j(\u_{j-1}^{t+1}) - \u_j^{t+1}}.
	\end{align*}
Then,
\begin{align*}
	\Norm{\nabla F(\w^{t+1}) - \prod_{k=1}^{K+1} \nabla f_k(\u_{k-1}^t)}^2 \leq K\left(\sum_{k=1}^K C_k^2\Upsilon_k^{t+1}\right),
\end{align*}
where $\Upsilon_k^t\coloneqq \Norm{f_k(\u_{k-1}^t)-\u_k^t}^2$, $C_k\coloneqq L_f^K L_g(1+L_f+\dotsc+L_f^{K-k})$.

	Based on the definition of $\Delta^t$, we have
\begin{align*}
	\E\left[\Delta^{t+1}\mid \F_t\right] &= \E\left[\Norm{\prod_{k=1}^{K+1} \nabla f_k(\u_{k-1}^{t+1}) - \hat{g}_1^{t+1} \prod_{k=2}^{K+1} \nabla f_k(\u_{k-1}^{t+1})}^2\mid \F_t\right]\\
	& \quad\quad + \E\left[\Norm{ \hat{g}_1^{t+1} \prod_{k=2}^{K+1} \nabla f_k(\u_{k-1}^{t+1}) - \hat{g}_1^{t+1} \hat{g}_2^{t+1}\prod_{k=3}^{K+1} \nabla f_k(\u_{k-1}^{t+1})}^2\mid \F_t\right]\\
	& \quad\quad \dotsc\\
	& \quad\quad +  \E\left[\Norm{\left(\prod_{k=1}^K\hat{g}_k^{t+1}\right)\nabla f_{K+1}(\u_K^{t+1}) - \left(\prod_{k=1}^K\hat{g}_k^{t+1}\right)\hat{\nabla} f_{K+1}(\u_K^{t+1})}^2\mid \F_t\right]\\
	&  \leq \frac{n^2\sigma_g^2 L_f^{2K}}{B} + \frac{n^4\sigma_g^4 L_f^{2(K-1)}L_f^{2(K-1)}}{B^2} + \dotsc + \frac{n^{2K}\sigma_g^{2K}L_f^2}{B^K} + \frac{n^{2K}\sigma_g^{2(K+1)}}{B^K}\\
	& = \underbrace{\sum_{k=1}^{K} \frac{n^{2k}L_f^{2(K+1-k)}\sigma^{2k}}{B^k}+ \frac{n^{2K}\sigma_g^{2(K+1)}}{B^K}}_{\coloneqq C_{\Delta}}.
\end{align*}
\end{proof}

\begin{lemma}
For $0< \beta_{0,k}\leq 1$ and $\Upsilon_k^t \coloneqq \Norm{\u_k^t - f_k(\u_{k-1}^t)}^2$, we have
\begin{align}\label{eq:fval_recursion}
	\E\left[\Upsilon_k^{t+1}\right] &\leq \left(1 - \frac{\beta_{0,k}B}{2n}\right)	\E\left[\Upsilon_k^t\right] + \beta_{0,k}^2 n \sigma_f^2 + \begin{cases}
		\frac{5\eta^2 c_u^2 n^2L_f^2\Norm{\m^t}^2}{B\beta_{0,k}}, & k=1\\
		\frac{5\beta_{0,k-1}^2 n^2\sigma_f^2 L_f^2}{\beta_{0,k}} + \frac{5\beta_{0,k-1}^2 n L_f^2}{\beta_{0,k}}\E\left[\Upsilon_{k-1}^t\right], & 2\leq k\leq K.
	\end{cases}
\end{align}
\end{lemma}
\begin{proof}
According to \eqref{eq:update_u}, we can derive that
\begin{align*}
	\E\left[\Upsilon_k^{t+1}\right] &= \E\left[\Norm{\u_k^{t+1} - f_k(\u_{k-1}^{t+1})}^2\right]\\
	& = \E\left[\sum_{k=1}^n\Norm{\u_{k,i}^{t+1} - f_{k,i}(\u_{k-1}^{t+1})}^2\right]\\
	& =  \E\left[\sum_{k=1}^n\frac{B}{n}\underbrace{\Norm{(1-\beta_{0,k})\u_{k,i}^t + \beta_{0,k} \hat{f}_{k,i}(\u_{k-1}^t) - f_{k,i}(\u_{k-1}^{t+1})}^2}_{\circled{5}} + \sum_{k=1}^n(1-\frac{B}{n})\underbrace{\Norm{\u_{k,i}^t - f_{k,i}(\u_{k-1}^{t+1})}^2}_{\circled{6}}\right]
\end{align*}
The first term on the R.H.S. can be bounded as
\begin{align*}
	& \E\left[\Norm{\circled{5}}^2\right] = 	\E\left[\Norm{(1-\beta_{0,k})(\u_{k,i}^t - f_{k,i}(\u_{k-1}^t)) + (f_{k,i}(\u_{k-1}^t) - f_{k,i}(\u_{k-1}^{t+1})) + \beta_{0,k} (\hat{f}_{k,i}(\u_{k-1}^t) - f_{k,i}(\u_{k-1}^t) )}^2\right] \\
	& \leq (1-\beta_{0,k})\E\left[\Norm{\u_{k,i}^t - f_{k,i}(\u_{k-1}^t)}^2\right] + \frac{2L_f^2}{\beta_{0,k}}\E\left[\Norm{\u_{k-1}^{t+1} - \u_{k-1}^t}^2\right]+ \beta_{0,k}^2 \sigma_f^2.
\end{align*}
If $\beta\leq \frac{n}{B}$, we have
\begin{align*}
	\E\left[\Norm{\circled{6}}^2\right] \leq \left(1+\frac{\beta_{0,k} B}{2n}\right)	\E\left[\Norm{\u_{k,i}^t - f_{k,i}(\u_{k-1}^t)}^2\right] + \frac{3nL_f^2}{\beta_{0,k}B}\E\left[\Norm{\u_{k-1}^{t+1} - \u_{k-1}^t}^2\right].
\end{align*}
Then,
\begin{align*}
	\E\left[\Upsilon_k^{t+1}\right] &\leq \left(1 - \frac{\beta_{0,k}B}{2n}\right)	\E\left[\Upsilon_k^t\right] + \frac{5n^2L_f^2}{B\beta_{0,k}}\E\left[\Norm{\u_{k-1}^{t+1} - \u_{k-1}^t}^2\right] + \beta_{0,k}^2 n \sigma_f^2.
\end{align*}
When $k=1$, we have $\E\left[\Norm{\u_{k-1}^{t+1} - \u_{k-1}^t}^2\mid \F_t\right] = \Norm{\w^{t+1} - \w^t}^2 \leq \eta^2c_u^2 \Norm{\m^t}^2$. When $2\leq k\leq K$, consider that $\u_{k-1,i}^{t+1} = \u_{k-1,i}^t$ if $i\notin \B_{k-1}^{t+1}$. 
\begin{align*}
& \E\left[\Norm{\u_{k-1}^{t+1} - \u_{k-1}^t}^2\right] \\
& = \E\left[\sum_{i\in\B_{k-1}^{t+1}}\Norm{\u_{k-1,i}^{t+1} - \u_{k-1,i}^t}^2\right] = \beta_{0,k-1}^2\E\left[\sum_{i\in\B_{k-1}^{t+1}}\Norm{\hat{f}_{k-1,i}(\u_{k-2}^t) - \u_{k-1,i}^t}^2\right]\\
& = \beta_{0,k-1}^2\E\left[\sum_{i\in\B_{k-1}^{t+1}}\Norm{\hat{f}_{k-1,i}(\u_{k-2}^t) - f_{k-1,i}(\u_{k-2}^t)}^2\right] + \beta_{0,k-1}^2\E\left[\sum_{i\in\B_{k-1}^{t+1}}\Norm{f_{k-1,i}(\u_{k-2}^t) - \u_{k-1,i}^t}^2\right]\\
& \leq \beta_{0,k-1}^2 B \sigma_f^2 +  \frac{\beta_{0,k-1}^2 B}{n}\E\left[\Norm{f_{k-1}(\u_{k-2}^t) - \u_{k-1}^t}^2\right] =   \beta_{0,k-1}^2 B \sigma_f^2 +  \frac{\beta_{0,k-1}^2 B}{n} \E\left[\Upsilon_{k-1}^t\right].
\end{align*}
Thus, we have
\begin{align*}
	\E\left[\Upsilon_k^{t+1}\right] &\leq \left(1 - \frac{\beta_{0,k}B}{2n}\right)	\E\left[\Upsilon_k^t\right] + \beta_{0,k}^2 n \sigma_f^2 + \begin{cases}
	 \frac{5\eta^2 c_u^2 n^2L_f^2\Norm{\m^t}^2}{B\beta_{0,k}}, & k=1\\
	  \frac{5\beta_{0,k-1}^2 n^2\sigma_f^2 L_f^2}{\beta_{0,k}} + \frac{5\beta_{0,k-1}^2 n L_f^2}{\beta_{0,k}}\E\left[\Upsilon_{k-1}^t\right], & 2\leq k\leq K.
	\end{cases}
\end{align*}
\end{proof}

\subsection{Proof of Theorem~\ref{thm:main}}
\begin{proof}
Take expectation on both sides of Lemma~\ref{lem:starter} and telescope the recursion from iteration 1 to $T$.
\begin{align}\label{eq:base}
	\sum_{t=1}^T\E\left[\Norm{\nabla F(\w^t)}^2\right] \leq \frac{2(F(\w^t) - F^*)}{\eta c_l} + \frac{c_u}{c_l}\sum_{t=1}^T \E\left[\Phi^t\right] - \frac{1}{2}\sum_{t=1}^T \E\left[\Norm{\m^t}^2\right], 
\end{align}
where $F^*$ is the global lower bound of $F(\w)$. Using the tower property of conditional expectation and telescoping \eqref{eq:grad_recursion} leads to
\begin{align}\label{eq:unrolled_grad_recursion}
	\sum_{t=1}^T \E\left[\Phi^t\right] & \leq \frac{\E\left[\Phi^1\right]}{\beta_1} + \frac{4\eta^2c_u^2 L_F^2}{\beta_1^2}\sum_{t=1}^T \E\left[\Norm{\m^t}^2\right] + 4K \sum_{t=1}^T \sum_{k=1}^K C_k^2\Upsilon_k^{t+1} + \beta_1 T C_\Delta.
\end{align}
Telescope \eqref{eq:fval_recursion} from iteration 1 to $T$.
\begin{align*}
	\sum_{t=1}^T\E\left[\Upsilon_k^t\right] & \leq \frac{2n\E\left[\Upsilon_k^1\right]}{B\beta_{0,k}} + \frac{2T n^2  \sigma_f^2\beta_{0,k} }{B} + \frac{10 T n^3\sigma_f^2 L_f^2\beta_{0,k-1}^2}{B \beta_{0,k}^2} + \frac{10 n^2 L_f^2 \beta_{0,k-1}^2}{B\beta_{0,k}^2} \sum_{t=1}^T \E\left[\Upsilon_{k-1}^t\right],\quad 2\leq k\leq K,\\
	\sum_{t=1}^T\E\left[\Upsilon_k^t\right] & \leq \frac{2n\E\left[\Upsilon_k^1\right]}{B\beta_{0,k}} + \frac{2T n^2  \sigma_f^2\beta_{0,k} }{B} + \frac{10 \eta^2c_u^2n^3L_f^2}{B^2\beta_{0,k}^2}\sum_{t=1}^T\Norm{\m^t}^2 ,\quad k=1.
\end{align*}
We can further derive that
\begin{align*}
\sum_{t=1}^T\E\left[\Upsilon_k^t\right] & \leq \sum_{\iota=1}^k \frac{2n\beta_{0,\iota}\E\left[\Upsilon_{\iota}^1\right]}{B\beta_{0,k}^2} \left(\frac{10n^2L_f^2}{B}\right)^{k- \iota} + \sum_{\iota=1}^k \frac{2T n^2\sigma_f^2 \beta_{0,\iota}^3}{B\beta_{0,k}^2}\left(\frac{10n^2L_f^2}{B}\right)^{k- \iota} \\
& \quad\quad + \sum_{\iota=2}^k \frac{10 T n^3 \sigma_f^2 L_f^2 \beta_{0,\iota-1}^2}{B\beta_{0,k}^2} \left(\frac{10n^2L_f^2}{B}\right)^{k - \iota} + \left(\frac{10n^2L_f^2}{B}\right)^{k-1}\frac{10\eta^2c_u^2n^3L_f^2}{B^2\beta_{0,k}^2}\sum_{t=1}^T \E\left[\Norm{\m^t}^2\right].
\end{align*}
Then, 
\begin{align*}
	& \sum_{t=1}^T \sum_{i=1}^K C_k^2\Upsilon_k^{t+1} \\
	& \leq \sum_{k=1}^K\sum_{\iota=1}^k \frac{2C_k^2 n\beta_{0,\iota}\E\left[\Upsilon_{\iota}^2\right]}{B\beta_{0,k}^2} \left(\frac{10n^2L_f^2}{B}\right)^{k- \iota} +  \sum_{k=1}^K\sum_{\iota=1}^k \frac{2C_k^2 T n^2\sigma_f^2 \beta_{0,\iota}^3}{B\beta_{0,k}^2}\left(\frac{10n^2L_f^2}{B}\right)^{k- \iota}\\
	& \quad\quad + \sum_{i=1}^K \sum_{\iota=2}^k \frac{10 C_k^2 T n^3 \sigma_f^2 L_f^2 \beta_{0,\iota-1}^2}{B\beta_{0,k}^2} \left(\frac{10n^2L_f^2}{B}\right)^{k- \iota} + \left(\sum_{t=2}^{T+1} \E\left[\Norm{\m^t}^2\right]\right)\sum_{k=1}^KC_k^2\left(\frac{10n^2L_f^2}{B}\right)^{k-1}\frac{10\eta^2c_u^2n^3L_f^2}{B^2\beta_{0,k}^2}.
\end{align*}
Plug the inequality above into \eqref{eq:unrolled_grad_recursion}.
\begin{align*}
		\sum_{t=1}^T \E\left[\Phi^t\right] & \leq \frac{\E\left[\Phi^1\right]}{\beta_1} + \frac{4\eta^2c_u^2 L_F^2}{\beta_1^2}\sum_{t=1}^T \E\left[\Norm{\m^t}^2\right] + \left(\sum_{t=2}^{T+1} \E\left[\Norm{\m^t}^2\right]\right)\sum_{k=1}^KC_k^2\left(\frac{40K n^2L_f^2}{B}\right)^{k-1}\frac{10\eta^2c_u^2n^3L_f^2}{B^2\beta_{0,k}^2}\\
	& \quad\quad +  \sum_{k=1}^K\sum_{\iota=1}^k \frac{8KC_k^2 n\beta_{0,\iota}\E\left[\Upsilon_{\iota}^2\right]}{B\beta_{0,k}^2} \left(\frac{10n^2L_f^2}{B}\right)^{k- \iota} +  \sum_{k=1}^K\sum_{\iota=1}^k \frac{8KC_k^2 T n^2\sigma_f^2 \beta_{0,\iota}^3}{B\beta_{0,k}^2}\left(\frac{10n^2L_f^2}{B}\right)^{k- \iota}\\
	& \quad\quad + \sum_{k=1}^K \sum_{\iota=2}^k \frac{40 K C_k^2 T n^3 \sigma_f^2 L_f^2 \beta_{0,\iota-1}^2}{B\beta_{0,k}^2} \left(\frac{10n^2L_f^2}{B}\right)^{k- \iota} + \beta_1 T C_\Delta .
\end{align*}
Next, we plug the inequality above into \eqref{eq:base} and divide $T$ on both sides.
\begin{align*}
	& \frac{1}{T}\sum_{t=1}^T\E\left[\Norm{\nabla F(\w^t)}^2\right] \\
	& \leq \frac{2(F(\w^t) - F^*)}{\eta T c_l} + \frac{c_u\E\left[\Phi^1\right]}{\beta_1 Tc_l }+  \sum_{k=1}^K\sum_{\iota=1}^k \frac{8KC_k^2 nc_u\beta_{0,\iota}\E\left[\Upsilon_{\iota}^2\right]}{c_l B\beta_{0,k}^2 T} \left(\frac{10n^2L_f^2}{B}\right)^{k- \iota}\\
	& 	\quad\quad - \frac{1}{2T}\sum_{t=1}^T \E\left[\Norm{\m^t}^2\right] + \frac{4\eta^2c_u^3 L_F^2}{c_l \beta_1^2 T}\sum_{t=1}^T \E\left[\Norm{\m^t}^2\right] + \left(\frac{1}{T}\sum_{t=2}^{T+1} \E\left[\Norm{\m^t}^2\right]\right)\sum_{k=1}^KC_k^2\left(\frac{40K n^2L_f^2}{B}\right)^{k-1}\frac{10\eta^2c_u^3n^3L_f^2}{c_lB^2\beta_{0,k}^2}\\
	& \quad\quad +  \sum_{k=1}^K\sum_{\iota=1}^k \frac{8Kc_uC_k^2 n^2\sigma_f^2 \beta_{0,\iota}^3}{c_l B\beta_{0,k}^2}\left(\frac{10n^2L_f^2}{B}\right)^{k- \iota}  + \sum_{k=1}^K \sum_{\iota=2}^k \frac{40 c_u K C_k^2 n^3 \sigma_f^2 L_f^2 \beta_{0,\iota-1}^2}{c_l B\beta_{0,k}^2} \left(\frac{10n^2L_f^2}{B}\right)^{k- \iota} + \frac{\beta_1 c_u C_\Delta}{c_l}.
\end{align*}
We can make $ \frac{1}{T}\sum_{t=1}^T\E\left[\Norm{\nabla F(\w^t)}^2\right]\leq \epsilon^2$ if we set $\eta = O(\epsilon^K)$, $\beta_1 = O(\epsilon^K)$, $\beta_2 \in (0,1)$, $\beta_{0,k} = O(\epsilon^{K-k})$, $1\leq k\leq K$, and $T = O(\epsilon^{-(K+2)})$.
\end{proof}

\end{document}